\documentclass[journal]{IEEEtran}

\usepackage[utf8]{inputenc}
\usepackage{amsmath,amsfonts}
\usepackage{amssymb}
\usepackage{mathtools}
\usepackage{algorithmic}
\usepackage{algorithm}
\usepackage{array}
\usepackage{subcaption}
\usepackage{booktabs}
\usepackage{tabularx}
\usepackage{textcomp}
\usepackage{stfloats}
\usepackage{url}
\usepackage{graphicx}
\usepackage{cite}
\usepackage{amsthm}
\theoremstyle{plain}
\newtheorem{theorem}{Theorem}[section]
\newtheorem{lemma}{Lemma}
\newtheorem{proposition}[theorem]{Proposition}
\theoremstyle{definition}
\newtheorem{definition}[theorem]{Definition}
\newtheorem{assumption}[theorem]{Assumption}

\renewcommand{\IEEEQEDclosed}{}
\newcolumntype{Y}{>{\raggedright\arraybackslash}X}
\usepackage{hyperref}
\begin{document}

\title{Dual-Flow Reinforcement Learning with State-Aware Exploration}

\author{Qijun Li, Zheng Fu, Qi Song, Yifei He, Weitao Zhou, Hao Gao, Kun Jiang, and Diange Yang
\thanks{This work was supported in part by the National Natural Science Foundation of China (52394264, 52472449, 52372414), Independent Research Project of the State Key Laboratory of Intelligent Green Vehicle and Mobility, Tsinghua University (No. ZZ-PY-20250408), the Tsinghua University-Toyota Joint Center, and Tsinghua University–SAIC GM Wuling Joint Research Center.}
\thanks{Qijun Li, Zheng Fu, Qi Song, Yifei He, Weitao Zhou, Kun Jiang, and Diange Yang are with the School of Vehicle and Mobility and the State Key Laboratory of Intelligent Green Vehicle and Mobility, Tsinghua University, Beijing 100084, China (e-mail: lqj25@mails.tsinghua.edu.cn; fuzheng039@gmail.com; qisong@link.cuhk.edu.cn; heyf25@mails.tsinghua.edu.cn; zhouwt801@gmail.com; jiangkun@tsinghua.edu.cn; ydg@mail.tsinghua.edu.cn).}
\thanks{Hao Gao is with the Nanjing University of Posts and Telecommunications, Nanjing, China (e-mail: tsgaohao@gmail.com).}
\thanks{(Corresponding author: Diange Yang, Zheng Fu.)}
\thanks{This work has been submitted to the IEEE for possible publication. Copyright may be transferred without notice, after which this version may no longer be accessible.}}

\markboth{}%
{Li \MakeLowercase{\textit{et al.}}: Dual-Flow Reinforcement Learning with State-Aware Exploration}

\maketitle

\begin{abstract}
In complex continuous-control reinforcement learning tasks, multimodal optimal actions often coincide with uncertain, multimodal return distributions, making reliable value estimation and multimodal exploration challenging.
Existing value estimation methods using unimodal Gaussians restrict expressiveness and yield biased estimates. Recent generative policies can represent multimodal actions but often collapse to a few modes and under-explore high-value areas of the action space.
Motivated by these challenges, we propose Dual-Flow RL, a unified actor-critic framework that jointly models a continuous return distribution and a multimodal policy distribution using conditional flow matching (CFM). This design supports reliable value estimation and sustained multimodal exploration.
To further enhance exploration, we introduce an Entropy-Covariance Exploration Regulator (ECER) that enables state-aware exploration regulation leveraging policy entropy and action-uncertainty covariance.
Experiments on DeepMind Control Suite and Humanoid-Bench show that Dual-Flow RL achieves state-of-the-art performance on most tasks, significantly outperforming prior diffusion-based and flow-based methods.
\end{abstract}


\section{Introduction}

\IEEEPARstart{R}{ecent} years have witnessed the widespread application of reinforcement learning (RL) to continuous control tasks \cite{wei2026scalable,10879126}, with broad applications to real-world domains such as industrial control \cite{liu2025ekg} and autonomous driving \cite{ning2026optimal,wijaya2024blockchain}. To obtain optimal action sequences, an RL model is expected to capture a multimodal policy distribution while ensuring diversity among high-value modes. However, traditional deterministic or diagonal Gaussian policies \cite{haarnoja2018soft} are inherently unimodal, which inevitably confines their sampling probability to low-value transitional regions between modes. Moreover, conventional value estimation methods in RL typically regress toward a single expected return \cite{fujimoto2018addressing, kumar2020conservative}, failing to represent multimodal structure of the value function and thus often converging to spurious local optima.
\begin{figure}[!t]
	\centering
	\includegraphics[width=0.49\textwidth]{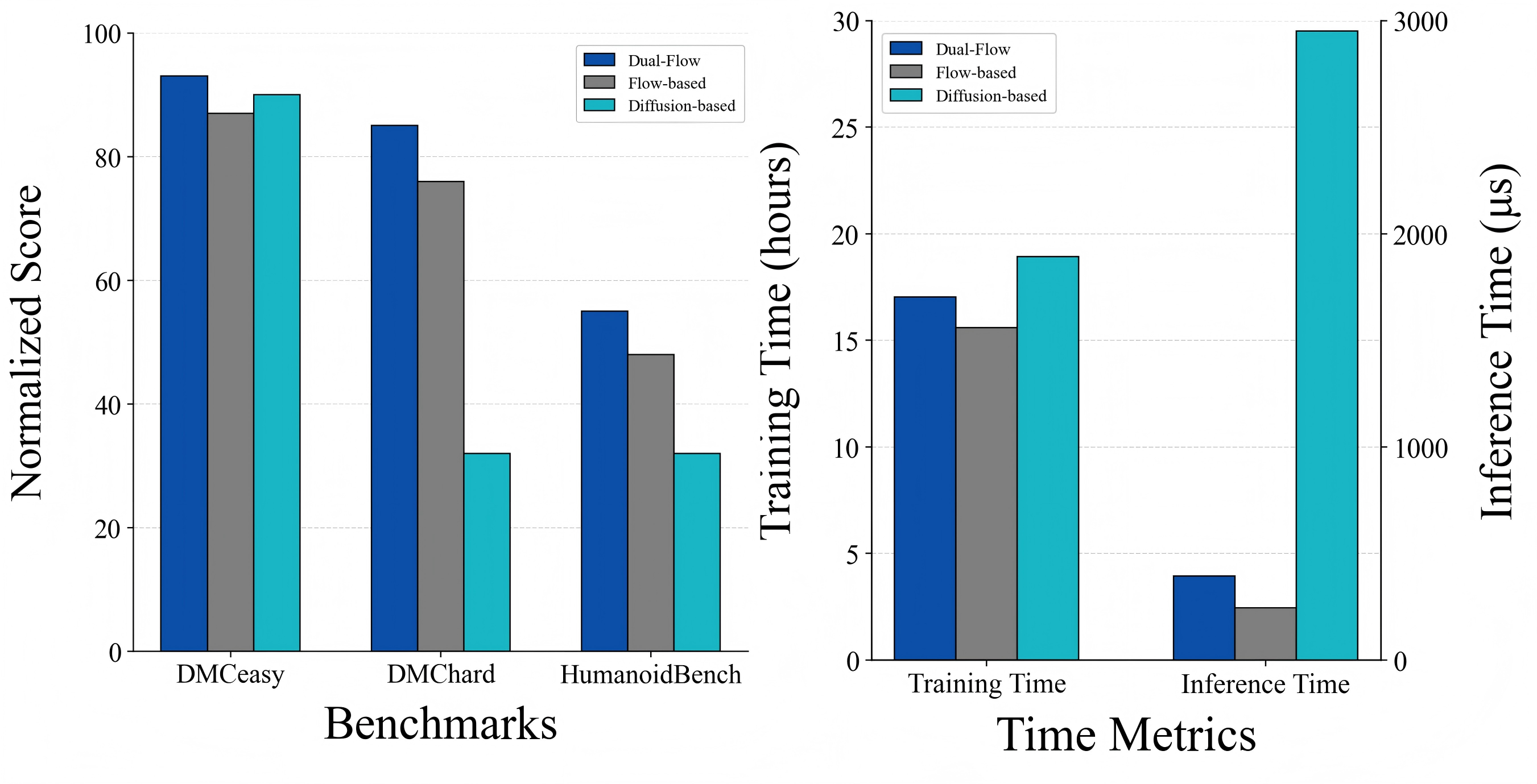}
	\caption{Comparison of performance and efficiency across benchmark groups.
		\textit{Left:} Normalized scores for Dual-Flow, flow-based (FlowRL), and diffusion-based (QVPO) methods on DMCeasy (the three Walker tasks), DMChard (other DMC tasks), and HumanoidBench.
		\textit{Right:} Computational efficiency on the \textit{Dog-run} task: 2M-step training time and single env step inference
		time.}
	\label{fig:norm_score}
\end{figure}

To address these limitations, recent work has introduced separate improvements, focusing on either the value estimation or the policy representation side alone. On the value side, a line of research employs parametric distribution families \cite{duan2021distributional} or quantile-based approximations \cite{dabney2018distributional} to model the distribution of returns, rather than a single expected value. However, such parametric approaches remain confined to a fixed distributional form (e.g., Gaussian), which limits their representational flexibility. On the policy side, recent studies try to parameterize complex action distributions via generative models, such as flow-based or diffusion-based architectures \cite{brahmanage2023flowpg, wang2024diffusion}. These methods enable effective modeling of multimodal policy distributions. Nevertheless, the exploration mechanisms used within these generative policies, like temperature tuning and entropy-guided noise scaling, are typically global and fixed, making it difficult to achieve state-adaptive and controllable exploration.
Overall, existing approaches are constrained in two key aspects: (1) limited expressive capacity
in value estimation, and (2) non-adaptive exploration mechanisms in generative policies.

Specifically, we present a unified framework that jointly models a continuous return distribution and a multimodal policy distribution via CFM. This design supports more accurate and reliable value estimation and sustained multimodal exploration. We further derive a value signal from the distributional critic to score actions and weight the alignment objective, guiding the policy toward actions preferred by the learned return distribution while preserving multimodal exploration.
To the best of our knowledge, this work is the first to jointly parameterize both the policy and the value function with flow-based models and train both via flow matching in a unified framework. Additionally, we propose an Entropy-Covariance Exploration Regulator (ECER) attached to the policy network to facilitate state-aware and controllable exploration. The ECER adaptively regulates state-aware exploration by integrating policy entropy and action-uncertainty covariance, thereby maintaining multimodal exploration and stabilizing the training process.

We evaluate our method on DeepMind Control Suite (DMC) \cite{tassa2018deepmind}, Humanoid-Bench (H-Bench) \cite{sferrazza2024humanoidbench} and MuJoCo Gym (MuJoCo) \cite{brockman2016openai}.
Fig.~\ref{fig:norm_score} provides an overview across the benchmark groups and shows that Dual-Flow consistently outperforms diffusion- and flow-based baselines.
We further conduct extensive evaluations on these benchmarks. Notably, Dual-Flow RL achieves state-of-the-art performance, surpassing FlowRL \cite{lv2025flow} by 31.6\% and SAC \cite{haarnoja2018soft} by 112.3\% on \textit{Humanoid-run} in DMC.

\section{Related Work}
\textbf{Distributional Critic in RL.}
Distributional critic reinforcement learning models the full return distribution rather than only the expected return. Early methods such as C51 \cite{bellemare2017distributional} approximate the return distribution with a categorical distribution over a fixed discrete support. Quantile-based methods like QR-DQN \cite{dabney2018distributional} and IQN \cite{dabney2018implicit} model return distributions via fixed or continuously sampled quantiles, but limited quantile resolution and truncation can introduce approximation errors in the distribution tails.
In continuous-action control, D4PG \cite{barth2018distributed} combines distributional critics with deterministic policy gradients and prioritized replay. In offline reinforcement learning, CODAC \cite{ma2021conservative} mitigates extrapolation error by imposing conservative penalties on quantile estimates of out-of-distribution actions.
A more lightweight line of work introduces explicit parametric assumptions on $Z(s,a)$ to simplify learning. For instance, DSAC \cite{duan2021distributional} models return distributions as Gaussians, which restricts expressiveness to unimodal moment matching and fails to capture multimodal or heavy-tailed returns.
In contrast, recent work has explored generative modeling approaches to directly learn return distributions. Bellman Diffusion employs diffusion models to represent continuous return distributions \cite{li2024bellman}, while normalizing-flow-based methods support unbounded support and offer greater expressive power for modeling multimodal and heavy-tailed structures \cite{kaddah2025flow}. Despite their expressiveness, these approaches typically rely on Gaussian base distributions or higher-order density field estimation, leading to increased computational overhead and greater sensitivity to hyperparameter choices.

\textbf{Generative Policies in RL.}
Generative policies aim to address the limited expressivity of deterministic policies in multimodal scenarios. Early work explored deep generative models such as VAEs, EBMs, and GANs to represent multimodal action distributions \cite{fujimoto2019off,kostrikov2021offline,ho2016generative}.
In recent years, diffusion-based policies have become widely used for modeling complex multimodal action distributions. A variety of methods integrate value information into the denoising process. Approaches such as QVPO \cite{ding2024diffusion} and QIPO \cite{zhang2025energy} leverage value signals to shape the evolving action distribution along the intermediate denoising steps. In contrast, others impose value constraints primarily on the final generated actions, as exemplified by Diffusion-QL \cite{wang2023diffusion}, QSM \cite{psenka2023learning} and DiffCPS \cite{he2024diffcps}.
More recently, flow-based models emerged as an effective tool for constructing expressive policy distributions. Flow Q-Learning (FQL) incorporates flow models as behavior-cloning priors for offline and offline-to-online settings \cite{park2025flow}, while FlowRL reparameterizes policies with flow models and trains them through constrained policy search and Wasserstein regularization \cite{lv2025flow}. ReinFlow \cite{zhang2025reinflow} introduces an online fine-tuning framework that injects learnable noise into deterministic generation paths to stabilize policy gradients, and SAC Flow \cite{zhang2025sac} examines gradient stability in off-policy learning and proposes sequential modeling and reparameterization strategies. Beyond control-focused reinforcement learning, Flow-GRPO, DanceGRPO, and MixGRPO extend flow and rectified-flow models to image and video generation \cite{liu2025flow,xue2025dancegrpo,li2025mixgrpo}.

\textbf{Exploration Regulation in RL.}
Exploration regulation in continuous-control RL is commonly implemented by controlling policy stochasticity. SAC \cite{haarnoja2018soft} regulates exploration through entropy regularization with an automatically adjusted temperature. gSDE \cite{raffin2020generalized} further introduces state-dependent exploration by modulating action noise through a state-conditioned transformation. Similar structured-exploration approaches also incorporate correlated or state-conditioned noise to shape exploration in continuous control \cite{chiappa2023latent, hollenstein2024colored}.
Building on these foundations, recent studies have investigated exploration regulation for generative policies. DACER \cite{wang2024diffusion} estimates action entropy by fitting a GMM to multiple samples per state to adaptively scale exploration noise. However, existing exploration regulation methods generally rely on global entropy schedules or fixed noise structures, offering only limited state-aware control and lacking fully closed-loop, adaptive adjustment.

\section{Preliminaries}
\subsection{Distributional Reinforcement Learning}
Notation. We consider an MDP $(\mathcal{S},\mathcal{A},P,R,\gamma)$ with continuous $\mathcal{S}$ and $\mathcal{A}$, transition density $P(s'|s,a)$, reward $R(s,a)$, and discount $\gamma\in[0,1)$. We write $r_t \triangleq R(s_t,a_t)$ and let $\pi(a|s)$ be a stochastic policy. We denote by $Z^\pi(s,a)$ the state--action return distribution under $\pi$.

The state--action return distribution under $\pi$ can be defined as
\begin{equation}
	\begin{aligned}
		&Z^\pi(s_t, a_t) \triangleq \sum_{k=0}^{\infty} \gamma^k r_{t+k}, \\
		&s_{t+1} \sim P(\cdot|s_t, a_t),  a_{t+1} \sim \pi(\cdot|s_{t+1}).
	\end{aligned}
\end{equation}
The conventional action-value function $Q^\pi$ is obtained as the expectation of the return random variable: $Q^\pi(s, a) = \mathbb{E}[Z^\pi(s, a)]$.

The distributional Bellman operator $\mathcal{T}_D^\pi$ is defined as:
\begin{equation}
	\begin{aligned}
		&(\mathcal{T}_D^\pi Z)(s, a) \stackrel{D}{=} r(s, a) + \gamma Z(s', a'),\\
		& s' \sim P(\cdot | s, a), a' \sim \pi(\cdot | s'),
	\end{aligned}
\end{equation}
where $\stackrel{D}{=}$ denotes equality in distribution.
The goal of distributional RL is to match the return-distribution critic $Z$ with the Bellman target distribution induced by the operator $\mathcal{T}_D^\pi$. Given replay data $\mathcal{D}$, the objective is formulated as:
\begin{equation}
	\begin{aligned}
		Z_{\text{new}} = \underset{Z \in \mathcal{Z}}{\arg \min} \mathbb{E}_{(s,a) \sim \mathcal{D}} \left[ d \big( Z(s,a), (\mathcal{T}^\pi Z^-)(s,a) \big) \right],
	\end{aligned}
\end{equation}
where $d(\cdot, \cdot)$ is a distributional discrepancy, and $Z^-$ denotes a delayed copy used to form a stable Bellman target for training.

\subsection{Flow Matching}
Flow Matching (FM) learns a time-dependent velocity field $v(x, t)$ that transports a base distribution $p_0$ to a target distribution $p_1$ over $t \in [0, 1]$. The objective is to construct a reference trajectory from the base distribution to the target distribution, governed by the continuity equation.

For paired samples $(x_0, x_1)$ with $x_0 \sim p_0$ and $x_1 \sim p_1$, we define the linear interpolation path:
\begin{equation}
	\begin{aligned}
		x_t = \psi_t(x_0, x_1), \quad \psi_t(x_0, x_1) = (1 - t)x_0 + tx_1.
	\end{aligned}
\end{equation}

The corresponding reference velocity is:
\begin{equation}
	\begin{aligned}
		u_t(x_0, x_1) = \frac{\partial}{\partial t} \psi_t(x_0, x_1).
	\end{aligned}
\end{equation}
To train the velocity field, Flow Matching minimizes the following loss function:
\begin{equation}
	\begin{aligned}
		\mathcal{L}_{FM}
		= \mathbb{E}_{\substack{t \sim \mathcal{U}[0,1] \\
				x_0 \sim p_0,\; x_1 \sim p_1}}\left[\left\| v(x_t,t) - u_t(x_0,x_1) \right\|_2^2\right].
	\end{aligned}
\end{equation}

Minimizing this objective enables the learned flow to transport samples from $p_0$ toward the target distribution $p_1$.

\section{Methodology}
Here, we detail our flow-based distributional critic and policy, both defined via probability-flow ODEs. We use the distributional critic to support policy improvement, and further introduce ECER to adapt exploration based on policy entropy and action-uncertainty covariance.
\subsection{Flow-based Distributional Critic and Policy}

For value evolution, we model the critic as a state--action conditional return distribution $Z(s, a)$ on the scalar return axis. For any $(s, a)$, we draw a base sample $z_0 \sim \mathcal{N}(0, 1)$ and evolve it over $t \in [0, 1]$ via ODE:
\begin{equation}
	\frac{dz_t}{dt} = v_z(s, a, z_t, t), \quad t \in [0, 1],
\end{equation}
where $v_z$ is the return velocity field. The return sample $z_1 \sim Z(s, a)$ is then obtained by:
\begin{equation}
	z_1 = z_0 + \int_0^1 v_z(s, a, z_t, t) dt.
\end{equation}
The terminal sample $z_1$ thus follows the return distribution $Z(s, a)$.

For the policy, we build a symmetric flow-based policy and introduce an exploration regulator to enhance exploration.

Given a state $s$, Policy-flow draws a base variable $a_0 \sim p_0 = \mathcal{N}(0, I)$ and evolves the action trajectory on $t \in [0, 1]$ via the flow ODE:
\begin{equation}
	\label{9}
	\frac{da_t}{dt} = v_{\pi}(s, a_t, t), \quad t \in [0, 1],
\end{equation}
where $v_{\pi}$ is the policy velocity field. The policy action $a_\pi$ is then defined as the terminal state $a_1$ through integration:
\begin{equation}
	\label{10}
	a_\pi = a_1 = a_0 + \int_0^1 v_{\pi}(s, a_t, t) dt.
\end{equation}

To modulate exploration, we introduce an Entropy-Covariance Exploration Regulator (ECER) atop the flow-based policy. ECER augments the deterministic mapping $a_\pi$ by applying a learned exploration scale $\sigma(s) \in \mathbb{R}^d_+$. Specifically, the agent executes an exploration-augmented action $a_e$:
\begin{equation}
	\label{11}
	\epsilon \sim \mathcal{N}(0, I), \quad a_e = a_\pi + \sigma(s) \odot \epsilon.
\end{equation}
The scale $\sigma(s)$ is produced by a learned regulator, trained using policy entropy and action-uncertainty covariance in Sec.~\ref{4.3}. This enables adaptive, state-aware control of exploration intensity.

\subsection{Flow-based Policy Iteration}
\label{sec:policy_iteration}
For policy evaluation, given any conditional return distribution $Z(\cdot \mid s, a)$ and a fixed policy $\pi$, we define the one-step distributional TD operator that maps $Z$ to a target distribution at $(s, a)$:
\begin{equation}
	\label{12}
	\begin{aligned}
		&(\mathcal{T}^\pi Z)(\cdot \mid s, a) \stackrel{d}{=} r(s, a) + \gamma Z(\cdot \mid s', a'), \\
		&s' \sim P(\cdot \mid s, a), a' \sim \pi(\cdot \mid s').
	\end{aligned}
\end{equation}
We define the resulting TD target distribution as:
\begin{equation}
	\label{13}
	p_{\text{td}}(\cdot \mid s, a) \doteq (\mathcal{T}^\pi Z)(\cdot \mid s, a).
\end{equation}
We evaluate the current policy by fitting $Z(\cdot \mid s, a)$ to the TD target distribution $p_{\text{td}}(\cdot \mid s, a)$ through CFM along the linear path between base noise and TD targets. Concretely, draw $z_0 \sim \mathcal{N}(0, 1)$, $t \sim \mathcal{U}[0,1]$, $y \sim p_{\text{td}}(\cdot \mid s, a)$, and define
\begin{equation}
	z_t = (1 - t)z_0 + ty, \quad u_z = y - z_0.
\end{equation}
The distributional evaluation objective is
\begin{equation}
	\label{15}
	\min \mathbb{E}_{(s, a, r, s') \sim \mathcal{D}} \mathbb{E}_{z_0, t, y} \left[ \| v_z(s, a, z_t, t) - u_z \|_2^2 \right].
\end{equation}
By optimizing Eq.~\eqref{15} to its optimum, we obtain a Bellman-consistent return distribution under the distributional TD operator $\mathcal{T}^\pi$, as stated in the following assumption and proposition.

\begin{assumption}[CFM realizability]
\label{ass:4.1}
There exists a velocity field $v_z^*(t, z)\in \mathcal{V}$ that attains the global minimum of the flow-matching loss.
\end{assumption}

\begin{assumption}[Lipschitz continuity for critic]
\label{ass:4.2}
The optimal velocity field $v_z^*(t, z)$ is continuous and bounded, and satisfies Lipschitz condition: $\|v_z(t, z) - v_z(t, z')\| \le L_{v}^{z} \|z - z'\| ,\forall t \in [0, 1], \forall z, z' \in \mathbb{R}.$
\end{assumption}

\begin{proposition}[Distributional Bellman consistency]
\label{prop:4.3}
Let $p_{\text{td}}(\cdot \mid s, a) \doteq (\mathcal{T}^\pi Z)(\cdot \mid s, a)$ be defined by Eqs.~\eqref{12}--\eqref{13}. If the distributional evaluation objective Eq.~\eqref{15} attains its global minimum, then the induced return distribution is Bellman-consistent under $\pi$:
$$Z(\cdot \mid s, a) = (\mathcal{T}^\pi Z)(\cdot \mid s, a).$$
\end{proposition}

\begin{IEEEproof}
See Appendix~\ref{proof-4.1}.
\end{IEEEproof}

For policy improvement, we use the mean value signal derived from the distributional critic. Specifically, let $\{z^{(k)}\}_{k=1}^K$ be samples drawn from $Z(s,a)$. Let $\mu(s,a)$ denote the empirical mean computed from these samples. The value used for policy improvement is defined as
\begin{equation}
	Q(s,a)\triangleq \mu(s,a)=\frac{1}{K}\sum_{k=1}^K z^{(k)}.
\end{equation}
Furthermore, we define the advantage of replay action $a$ over the current policy-flow action $a_{\pi}(s)$ and assign each $(s,a)$ a bounded exponential weight:
\begin{equation}
	\begin{aligned}
	&\Delta(s,a) \doteq \max\!\left(Q(s,a)-Q(s,a_{\pi}(s)),0\right),
	\quad \\
	&w(s,a)= \exp\!\left(\Delta(s,a)-\bar{\Delta}\right),
	\end{aligned}
\end{equation}
where $\bar{\Delta}$ is the mini-batch mean of $\Delta$. We construct the interpolation path between sample $a_0$ and the reference action $a$:
\begin{equation}
	a_t = (1 - t)a_0 + t a,  t \sim \mathcal{U}[0,1], u_{\pi}=a-a_0,
\end{equation}
where $a_0 \sim \mathcal{N}(0, I)$.
The policy training objective can be expressed as:
\begin{equation}
	\label{19}
	\begin{aligned}
		\min\; \mathbb{E}_{(s,a) \sim \mathcal{D}} \Big[
		-Q(s, a_{\pi}(s)) + w(s, a)\,\| v_{\pi}(s, a_t, t) - u_{\pi} \|_2^2
		\Big].
	\end{aligned}
\end{equation}
Objective~\eqref{19} comprises two parts. The former drives the policy toward higher-value actions, while the latter guides the update by aligning the policy distribution with high-quality behavior actions in the replay buffer.

\subsection{Entropy-Covariance Exploration Regulator}
\label{4.3}
We adapt exploration intensity using two diagnostics: policy entropy and action-uncertainty covariance. Policy entropy reflects how dispersed the action distribution is, while the covariance diagnostic quantifies the covariance between action-density concentration and return uncertainty.

To estimate policy entropy, we fit a $K_m$-component diagonal GMM $q(a|s)$ to $N$ actions sampled from buffer states $\{s_b\}_{b=1}^B$ and approximate the state-conditioned entropy $H(s)$ following \cite{huber2008entropy}. Detailed derivations are provided in Appendix~\ref{Derivation-1}. Next, we introduce an action-uncertainty covariance.

\begin{definition}[Action-Uncertainty Covariance]
\label{def:4.6}
For a state $s$, $\{a_i\}_{i=1}^N \sim \pi(\cdot \mid s)$. Let $\log q(a \mid s)$ denote the log-density evaluated under the fitted $K_m$-component diagonal GMM $q(a \mid s)$. We define the standardized action-uncertainty covariance as
\begin{align}
	 &\rho(s)\triangleq \mathrm{corr}\!\left(\log q(a\mid s),\, \sigma(s,a)\right), \quad \\
	 &\sigma(s,a) = \sqrt{
		\frac{1}{K}
		\sum_{k=1}^{K}
		\left(z^{(k)}-\mu(s,a)\right)^2,
	}
\end{align}
where ${\sigma}(s,a)$ is the uncertainty of the return distribution. Here $\operatorname{corr}(\cdot,\cdot)$ denotes the Pearson correlation computed over $\{a_i\}_{i=1}^N$.
$\rho(s)$ characterizes the concentration of the exploration distribution on high-uncertainty action regions.
\end{definition}

To regulate the exploration strength, we construct two gates from the policy entropy and the action-uncertainty covariance diagnostic.
First, the entropy gate and covariance gate are defined jointly as
\begin{equation}
	\begin{aligned}
		g_H = \exp(\max(0, H_{\text{tgt}} - \hat{H})), \quad g_D = \exp(\rho -  \rho_{\text{thr}}),
	\end{aligned}
\end{equation}
where $H_{\text{tgt}} = -c \cdot d$ and $d$ denotes the action dimension. These components are integrated into a combined gate and effective regularization mechanism, where the global gate $g = g_H g_D$ determines the effective regularization coefficient:
\begin{equation}
	\label{eff}
	\begin{aligned}
	\lambda_{\text{eff}} = \frac{\lambda_0}{g^2 + \epsilon},
	\end{aligned}
\end{equation}
where $\lambda_0$ is the base regularization coefficient. $g_H$, $g_D$, and $\lambda_{\text{eff}}$ are updated periodically during training.
The final ECER objective is formulated to maximize the gain-weighted log-likelihood while regularizing the exploration scale:
\begin{equation}
	\label{ecer}
	\begin{split}
		&\mathcal{L}_{\text{ECER}} = -\mathbb{E}_{s \sim \mathcal{D}, a_e \sim p(\cdot|s)} \left[ A_e(s) \log p(a_e | s) \right] \\
		& + \lambda_{\text{eff}}\mathbb{E}_{s \sim \mathcal{D}} \left[ \|\sigma(s)\|_2^2 \right].
	\end{split}
\end{equation}
Gain $A_e(s)$ is defined as the advantage of the interaction action $a_e$ relative to the reference policy $a_{\pi}(s)$:
\begin{equation}
	\begin{aligned}
		A_e(s) \triangleq Q(s, a_e) - Q(s, a_{\pi}(s)).
	\end{aligned}
\end{equation}
$\log p(a_e|s)$ is the log-likelihood of a diagonal Gaussian distribution $\mathcal{N}(a_{\pi}(s), \sigma(s)^2 I)$. Together, the entropy gate $g_H$ and covariance gate $g_D$ enable closed-loop regulation of the exploration scale.

\subsection{Practical Dual-Flow RL Algorithm}
To obtain a practical algorithm, we adopt parameterized function approximation for the distributional critic, the flow-based policy, and the exploration regulator. Specifically, we model the return distribution with a conditional flow $Z_\theta(\cdot \mid s, a)$ and parameterize the return and policy velocity fields as $v_{z, \theta}(s, a, z, t)$ and $v_{\pi, \phi}(s, a, t)$, respectively. We further learn a state-aware exploration scale $\sigma_{\psi}(s)$ for ECER.

For approximate policy evaluation, we minimize the distributional Bellman residual by fitting $Z_\theta(\cdot \mid s, a)$ to the one-step TD target distribution $p_{\text{td}}(\cdot \mid s, a) \triangleq (T^\pi Z_\theta)(\cdot \mid s, a)$ using the flow-matching objective in Sec.~\ref{sec:policy_iteration}:
\begin{equation}
	\begin{aligned}
		\mathcal{L}_Q(\theta) = \mathbb{E}_{(s, a, r, s') \sim \mathcal{D}} \mathbb{E}_{z_0, t, y} \left[ \left\| v_{z, \theta}(s, a, z_t, t) - u_z \right\|_2^2 \right],
	\end{aligned}
\end{equation}
where $y \sim p_{\text{td}}(\cdot \mid s, a)$. For policy improvement, we update the policy flow by minimizing the objective in Sec.~\ref{sec:policy_iteration},
\begin{equation}
	\begin{aligned}
		&\mathcal{L}_\pi(\phi) = \mathbb{E}_{(s, a) \sim \mathcal{D}} \mathbb{E}_{a_0, t} \Big[ - Q(s, a_\pi(s)) \\
		 &+ w(s, a) \left\| v_{\pi, \phi}(s, a_t, t) - u_{\pi} \right\|_2^2 \Big].
	\end{aligned}
\end{equation}
Finally, ECER is trained to increase the expectation of the return distribution while adaptively regulating the exploration intensity via the entropy-covariance gate $g$ and the effective coefficient $\lambda_{\text{eff}}$:
\begin{equation}
	\begin{aligned}
		&\mathcal{L}_{\text{ECER}}(\psi) =  - \mathbb{E}_{s \sim \mathcal{D}, a_{e} \sim p_\psi(\cdot \mid s)} \left[ A_{e}(s) \log p_\psi(a_{e} \mid s) \right] \\
		& + \lambda_{\text{eff}}(t) \mathbb{E}_{s \sim \mathcal{D}} \left[ \left\| \sigma_\psi(s) \right\|_2^2 \right].
	\end{aligned}
\end{equation}
The full training loop is given in Alg.~\ref{Dual-Flow}.

\begin{figure*}[t]
	\centering

	% ================= Left block: legend + (a)(b) =================
	\begin{minipage}[t]{0.49\textwidth}
		\centering
		\includegraphics[width=0.92\linewidth]{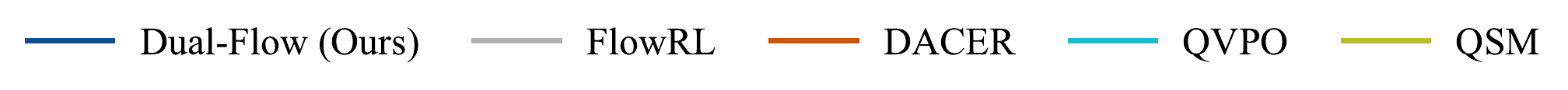}

		\vspace{0.6em}

		\begin{subfigure}[t]{0.49\linewidth}
			\centering
			\includegraphics[width=\linewidth]{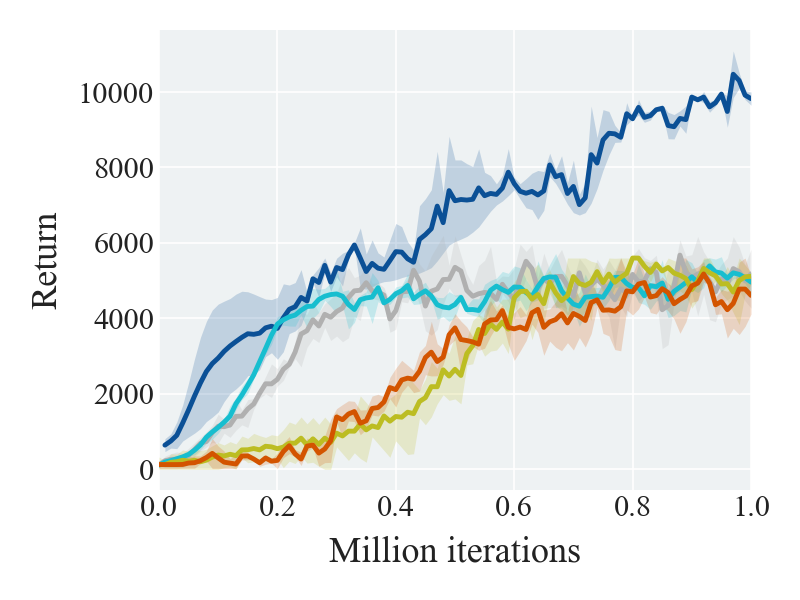}
			\caption{Humanoid-v3}
			\label{fig:Humanoidv3}
		\end{subfigure}\hfill
		\begin{subfigure}[t]{0.49\linewidth}
			\centering
			\includegraphics[width=\linewidth]{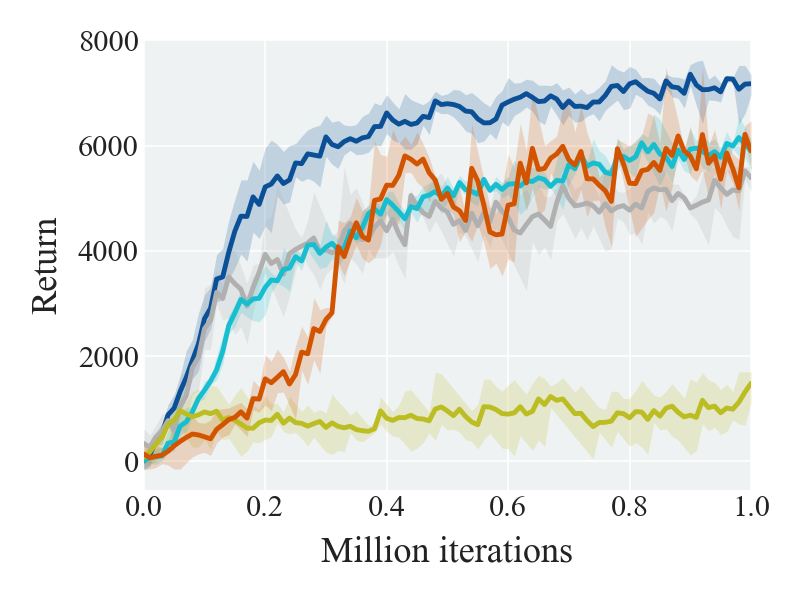}
			\caption{Ant-v3}
			\label{fig:Antv3}
		\end{subfigure}
	\end{minipage}\hfill
	% ================= Right block: legend + (c)(d) =================
	\begin{minipage}[t]{0.49\textwidth}
		\centering
		\includegraphics[width=0.92\linewidth]{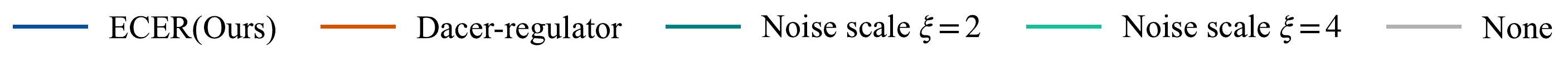}

		\vspace{0.6em}

		\begin{subfigure}[t]{0.49\linewidth}
			\centering
			\includegraphics[width=\linewidth]{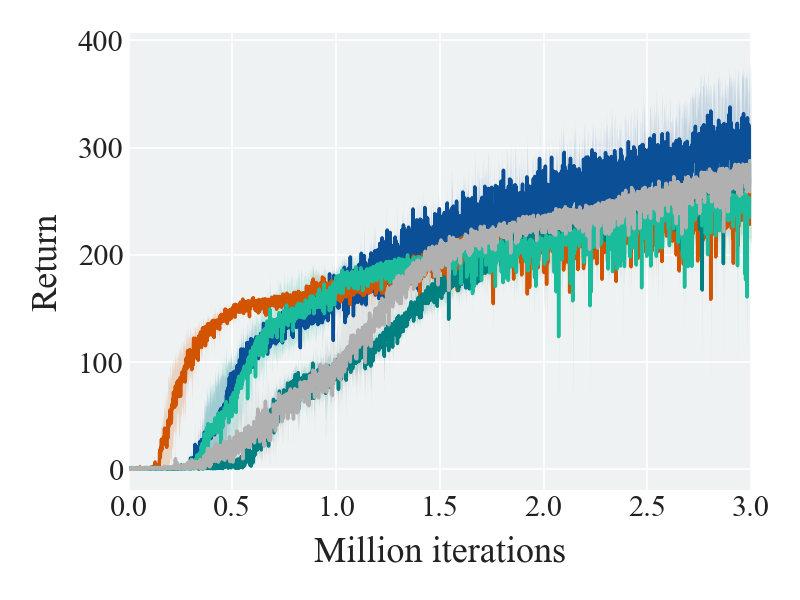}
			\caption{Training curves}
			\label{fig:ablation-exploration-train}
		\end{subfigure}\hfill
		\begin{subfigure}[t]{0.49\linewidth}
			\centering
			\includegraphics[width=\linewidth]{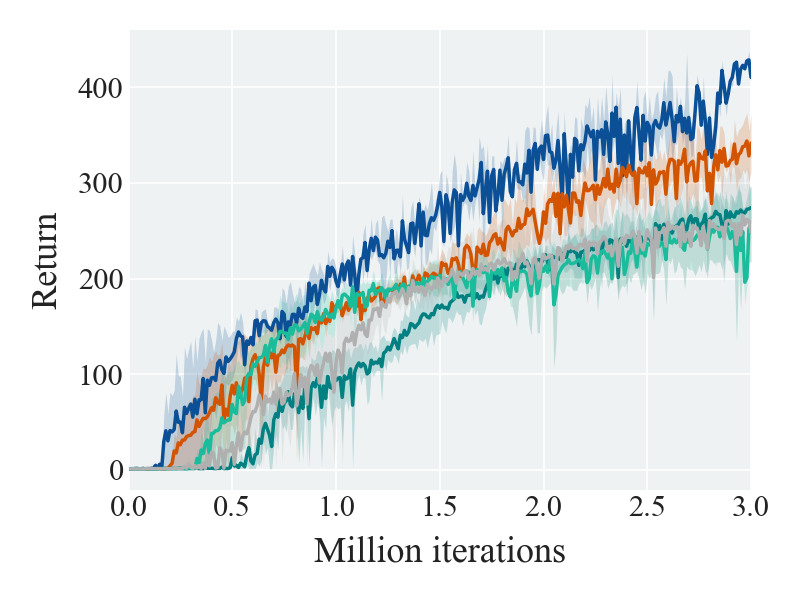}
			\caption{Evaluation curves}
			\label{fig:ablation-exploration-test}
		\end{subfigure}
	\end{minipage}

	\caption{
		\textbf{Ablations.}
		\textbf{(a--b) Evaluation curves on MuJoCo:}
		We compare Dual-Flow with various diffusion and flow baselines on \textit{Humanoid-v3} and \textit{Ant-v3}.
		Dual-Flow outperforms most baselines and achieves stronger overall performance.
		\textbf{(c--d) Exploration-strength regulation:} Dual-Flow on \textit{Humanoid-run} with different exploration regulators, including \textbf{ECER} (ours), \textbf{DACER-style entropy regulator}, and \textbf{fixed initial noise scale tuning} ($\xi$). \textbf{None} denotes Dual-Flow without any exploration regulator.
	}
	\label{fig:ablation_combined_nav_explore}
\end{figure*}

\begin{figure*}[t]
	\centering
	\includegraphics[width=\textwidth]{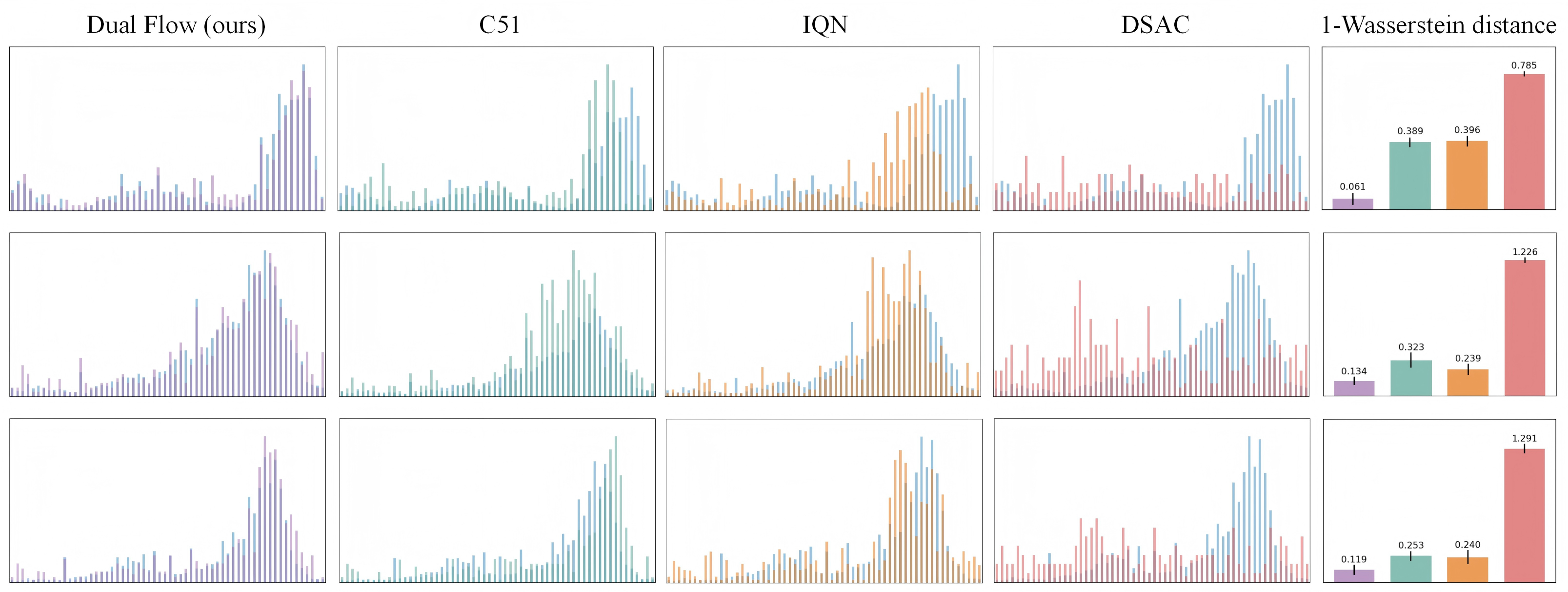}
	\caption{
		\textbf{Visualizing the return distribution.}
		Columns from left to right show the predicted return distributions of Dual-Flow, C51, IQN, and DSAC, followed by the corresponding 1-Wasserstein distance.
		Dual-Flow better matches the ground-truth distributions and achieves the lowest error.
	}
	\label{fig:dogrun_return_dist}
\end{figure*}

\begin{figure*}[h]
	\centering
	
	\includegraphics[width=0.95\linewidth]{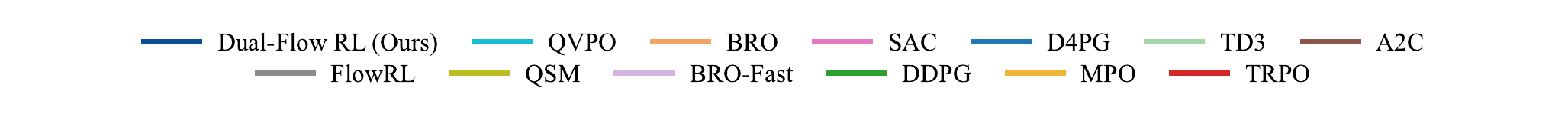}
	\vspace{-0.18cm}
	
	% Row 1 (Humanoid-run, Humanoid-stand, Dog-run, Dog-trot)
	\begin{minipage}[b]{0.24\linewidth}
		\centering
		\includegraphics[width=\linewidth]{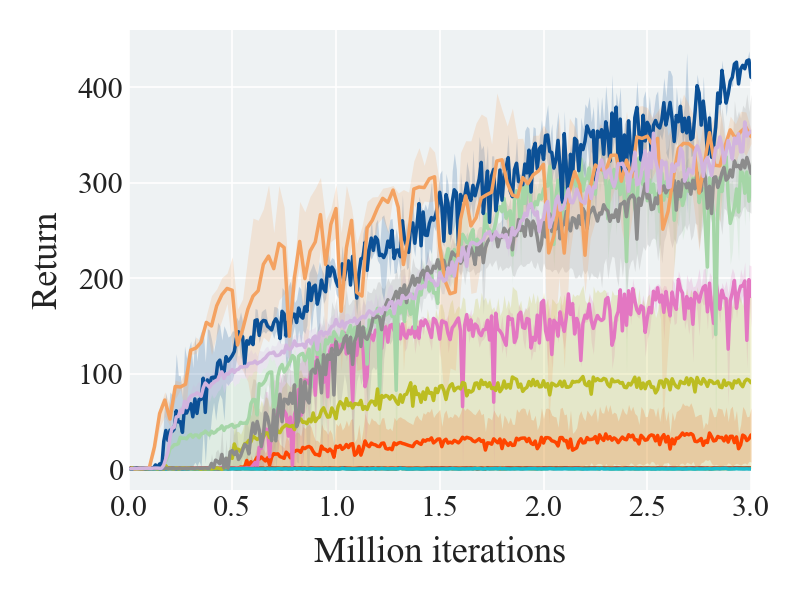}
		\subcaption{Humanoid-run}
		\label{fig:app_humanoid_run}
	\end{minipage}\hfill
	\begin{minipage}[b]{0.24\linewidth}
		\centering
		\includegraphics[width=\linewidth]{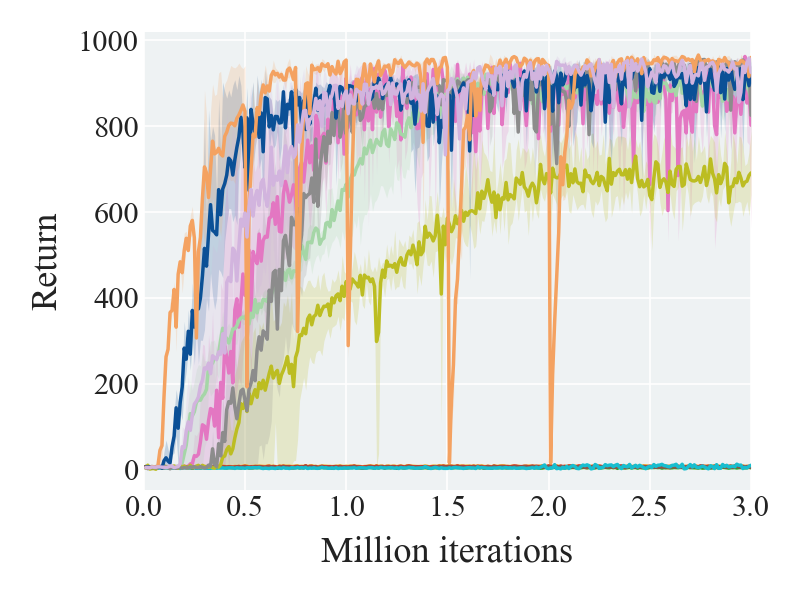}
		\subcaption{Humanoid-stand}
		\label{fig:app_humanoid_stand}
	\end{minipage}\hfill
	\begin{minipage}[b]{0.24\linewidth}
		\centering
		\includegraphics[width=\linewidth]{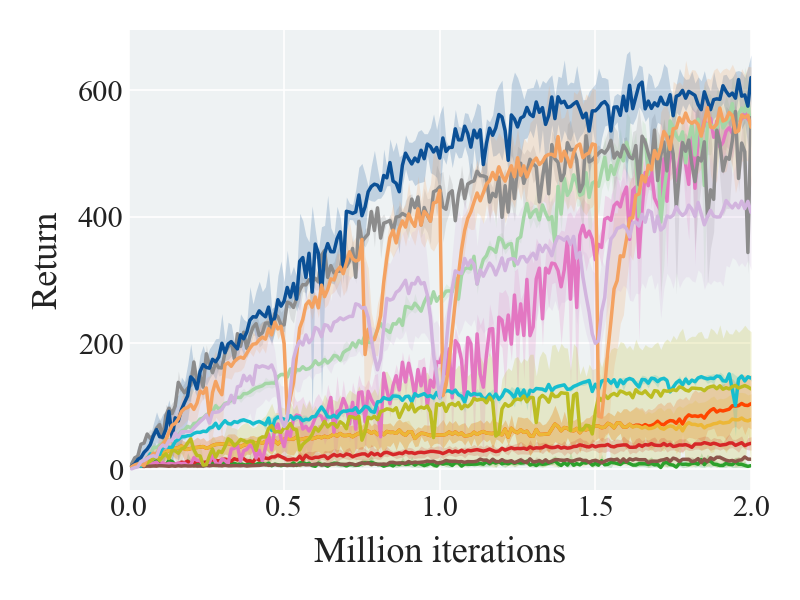}
		\subcaption{Dog-run}
		\label{fig:app_dog_run}
	\end{minipage}\hfill
	\begin{minipage}[b]{0.24\linewidth}
		\centering
		\includegraphics[width=\linewidth]{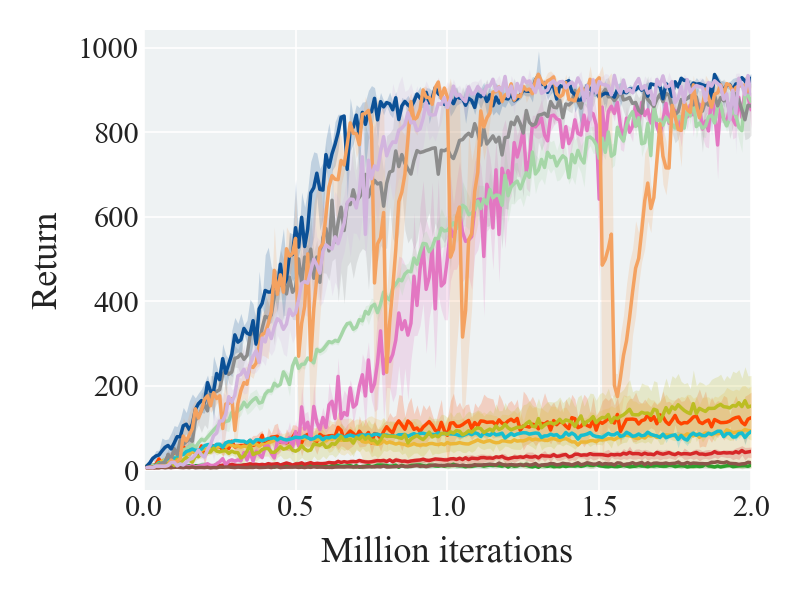}
		\subcaption{Dog-trot}
		\label{fig:app_dog_trot}
	\end{minipage}
	
	\vspace{0.85em}
	
	% Row 2 (Dog-stand, Dog-walk, Walker-stand, Walker-walk)
	\begin{minipage}[b]{0.24\linewidth}
		\centering
		\includegraphics[width=\linewidth]{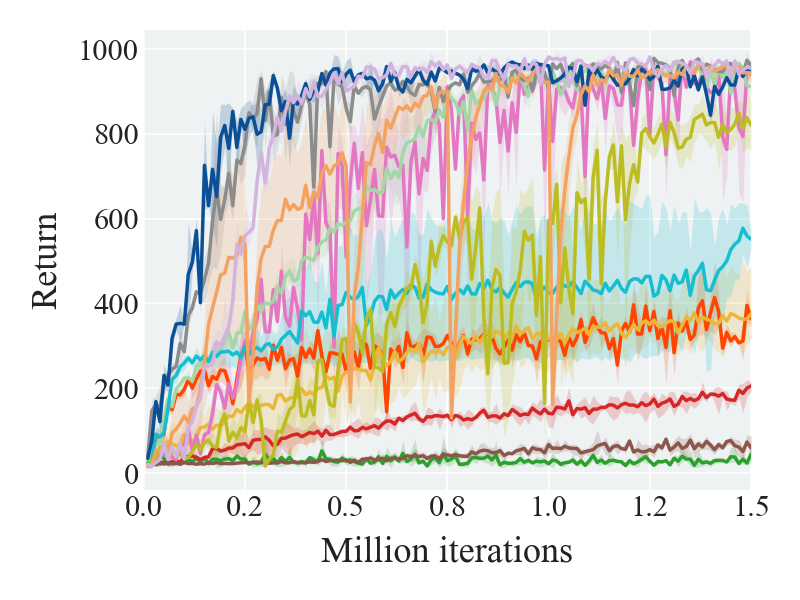}
		\subcaption{Dog-stand}
		\label{fig:app_dog_stand}
	\end{minipage}\hfill
	\begin{minipage}[b]{0.24\linewidth}
		\centering
		\includegraphics[width=\linewidth]{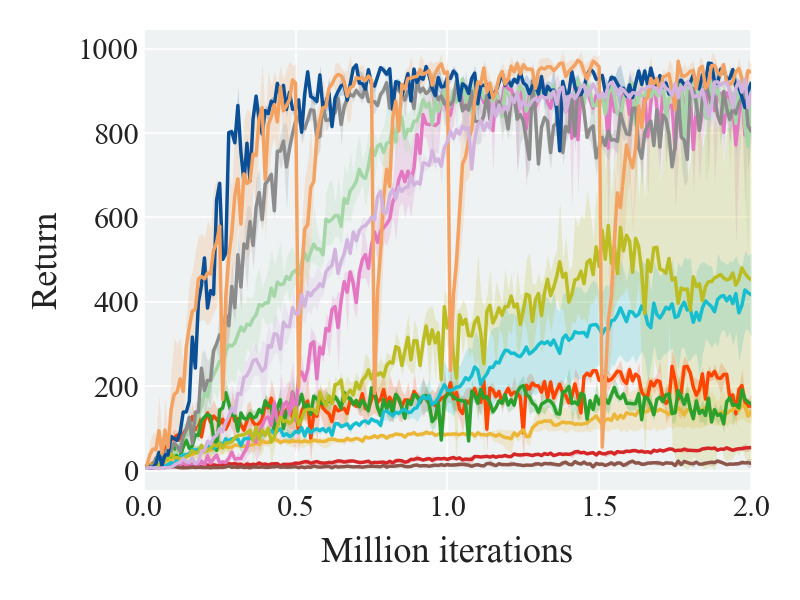}
		\subcaption{Dog-walk}
		\label{fig:app_dog_walk}
	\end{minipage}\hfill
	\begin{minipage}[b]{0.24\linewidth}
		\centering
		\includegraphics[width=\linewidth]{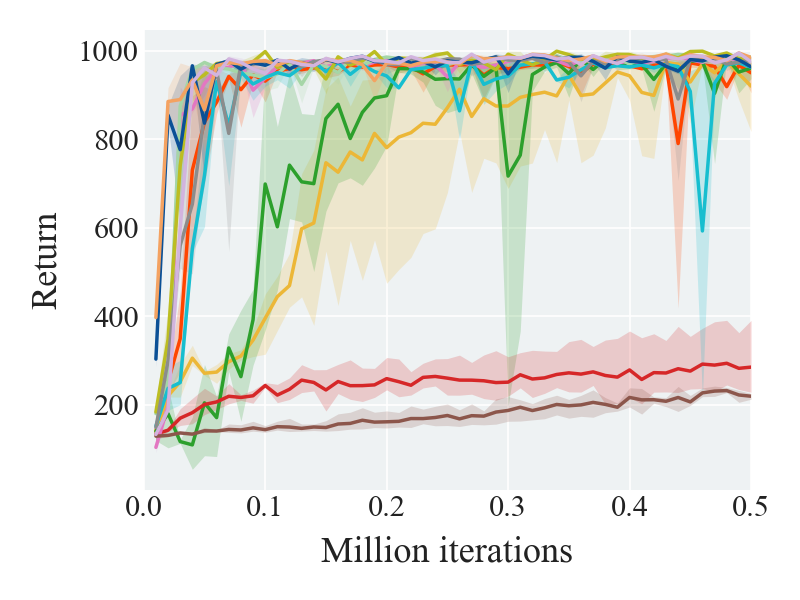}
		\subcaption{Walker-stand}
		\label{fig:app_walker_stand}
	\end{minipage}\hfill
	\begin{minipage}[b]{0.24\linewidth}
		\centering
		\includegraphics[width=\linewidth]{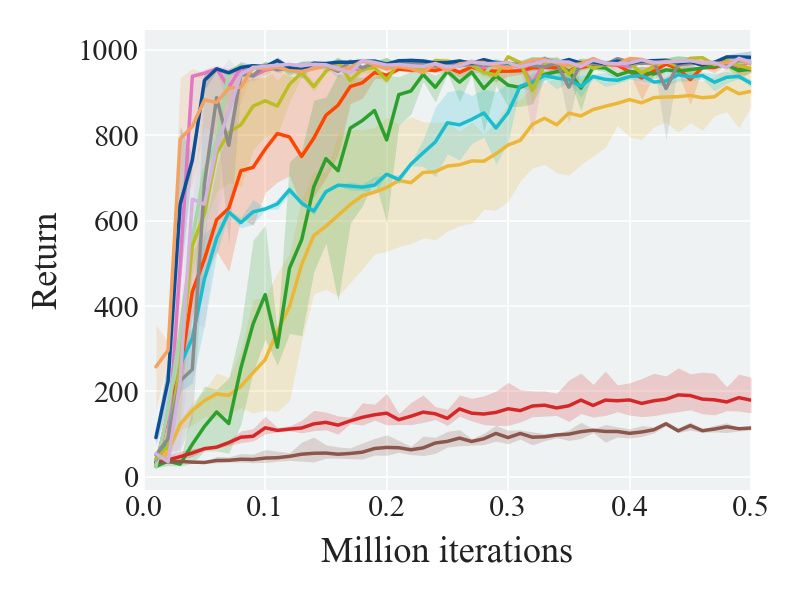}
		\subcaption{Walker-walk}
		\label{fig:app_walker_walk}
	\end{minipage}
	
	\vspace{0.85em}
	
	% Row 3 (Walker-run, Quadruped-walk, H1balance_simple, H1sit_hard)
	\begin{minipage}[b]{0.24\linewidth}
		\centering
		\includegraphics[width=\linewidth]{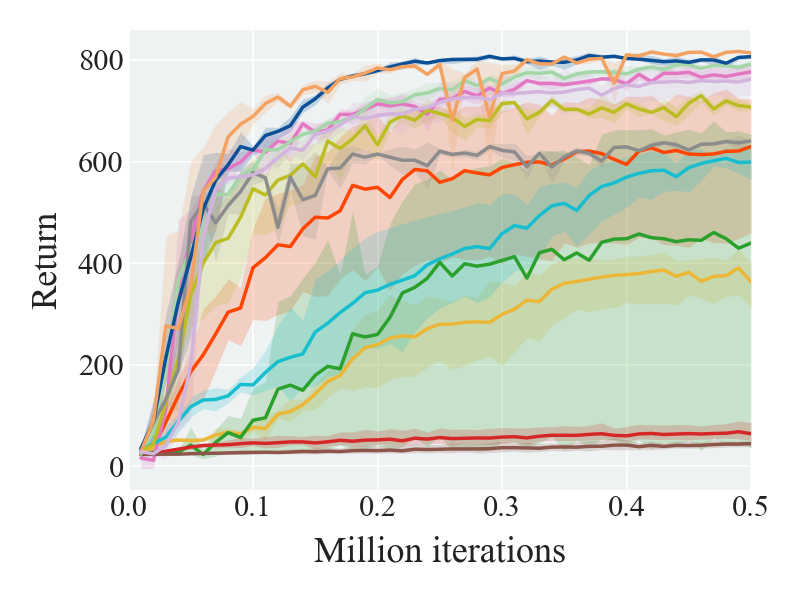}
		\subcaption{Walker-run}
		\label{fig:app_walker_run}
	\end{minipage}\hfill
	\begin{minipage}[b]{0.24\linewidth}
		\centering
		\includegraphics[width=\linewidth]{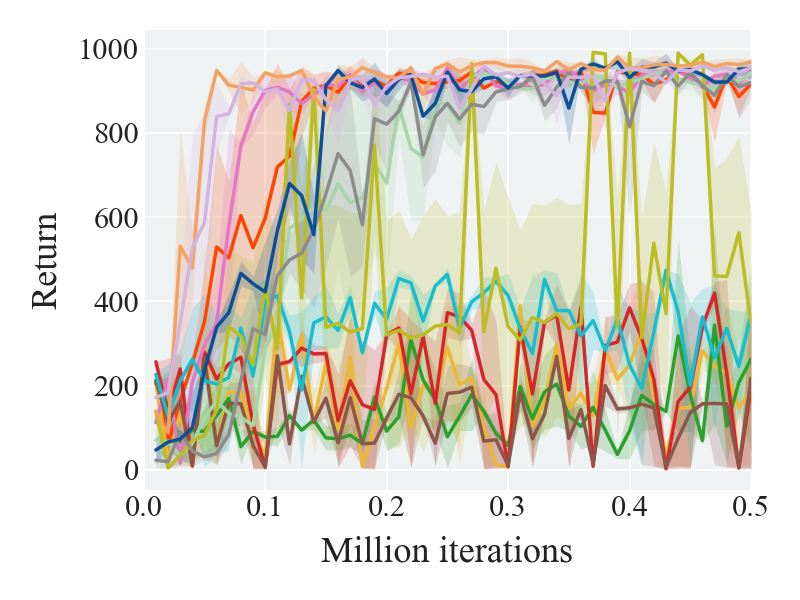}
		\subcaption{Quadruped-walk}
		\label{fig:app_quad_walk}
	\end{minipage}\hfill
	\begin{minipage}[b]{0.24\linewidth}
		\centering
		\includegraphics[width=\linewidth]{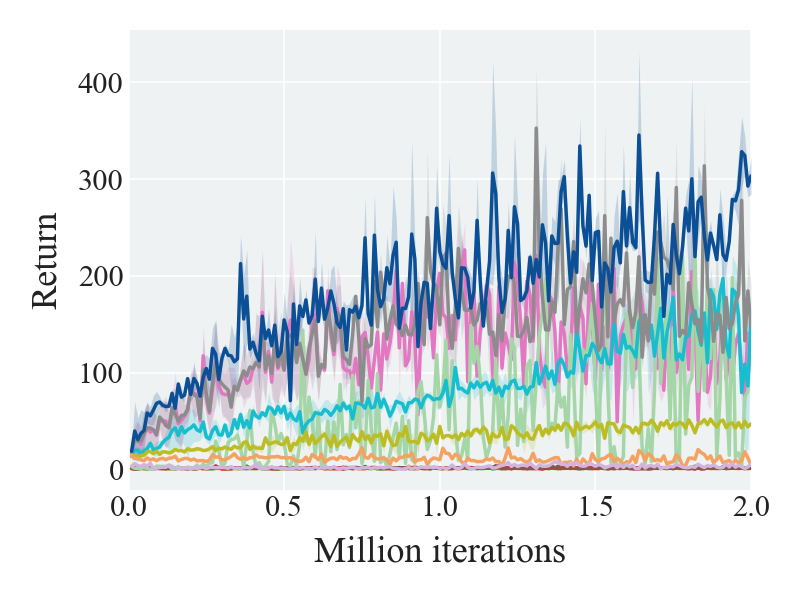}
		\subcaption{H1balance\_simple}
		\label{fig:app_h1_balance}
	\end{minipage}\hfill
	\begin{minipage}[b]{0.24\linewidth}
		\centering
		\includegraphics[width=\linewidth]{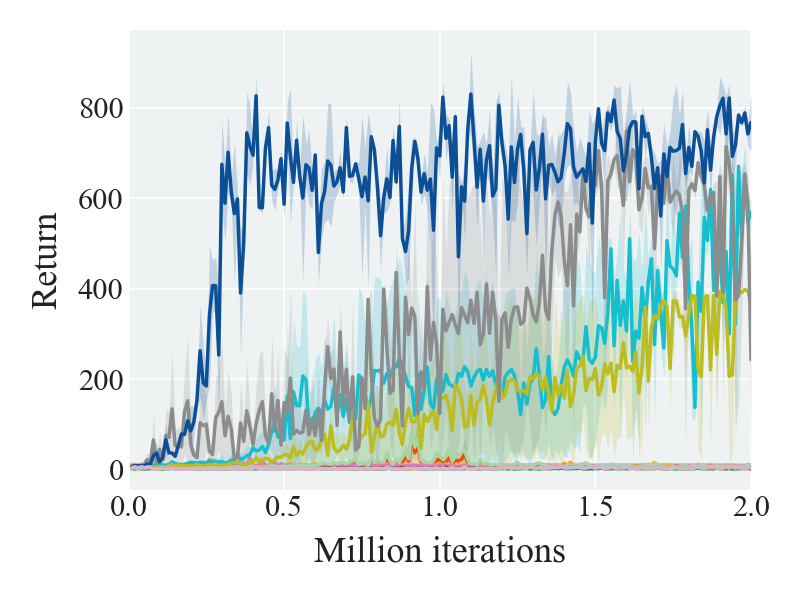}
		\subcaption{H1sit\_hard}
		\label{fig:app_h1_sit}
	\end{minipage}
	
	\vspace{0.85em}
	
	% Row 4 (H1balance_hard, H1crawl, H1maze, H1reach)
	\begin{minipage}[b]{0.24\linewidth}
		\centering
		\includegraphics[width=\linewidth]{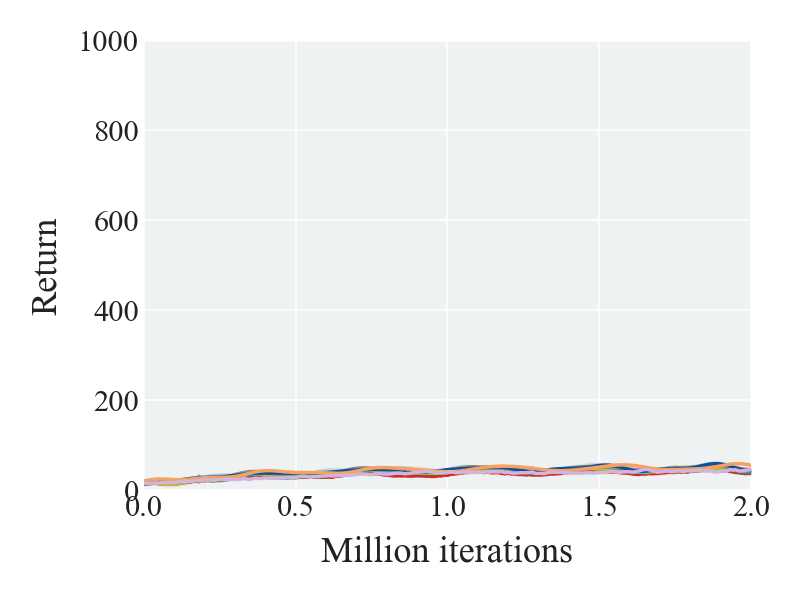}
		\subcaption{H1balance\_hard}
		\label{fig:app_h1_balance_hard}
	\end{minipage}\hfill
	\begin{minipage}[b]{0.24\linewidth}
		\centering
		\includegraphics[width=\linewidth]{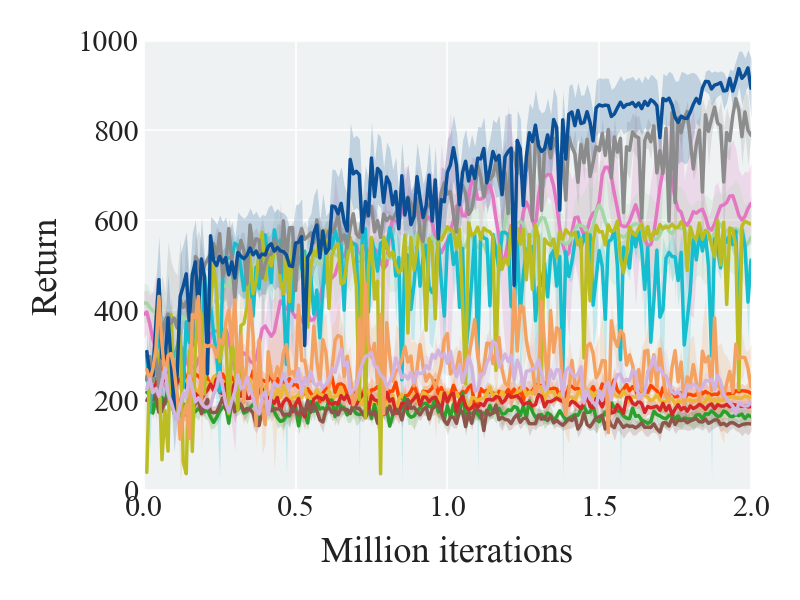}
		\subcaption{H1crawl}
		\label{fig:app_h1_crawl}
	\end{minipage}\hfill
	\begin{minipage}[b]{0.24\linewidth}
		\centering
		\includegraphics[width=\linewidth]{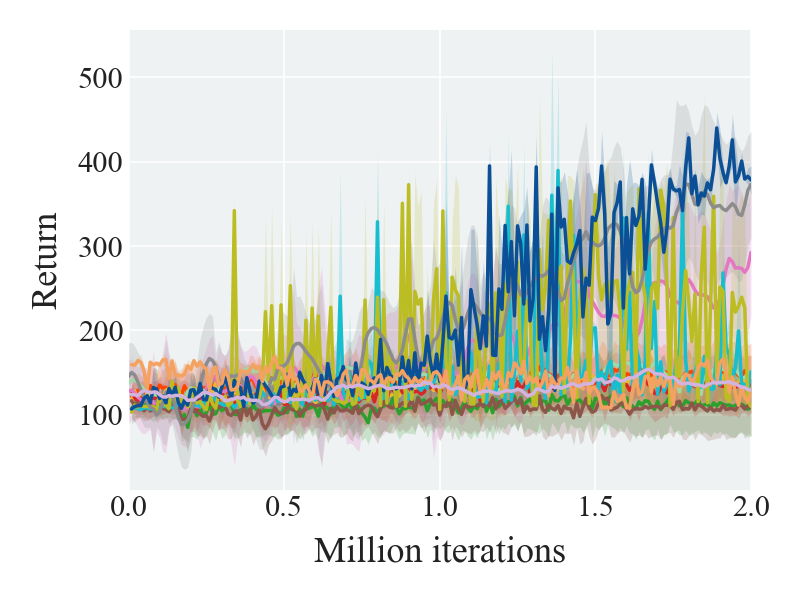}
		\subcaption{H1maze}
		\label{fig:app_h1_maze}
	\end{minipage}\hfill
	\begin{minipage}[b]{0.24\linewidth}
		\centering
		\includegraphics[width=\linewidth]{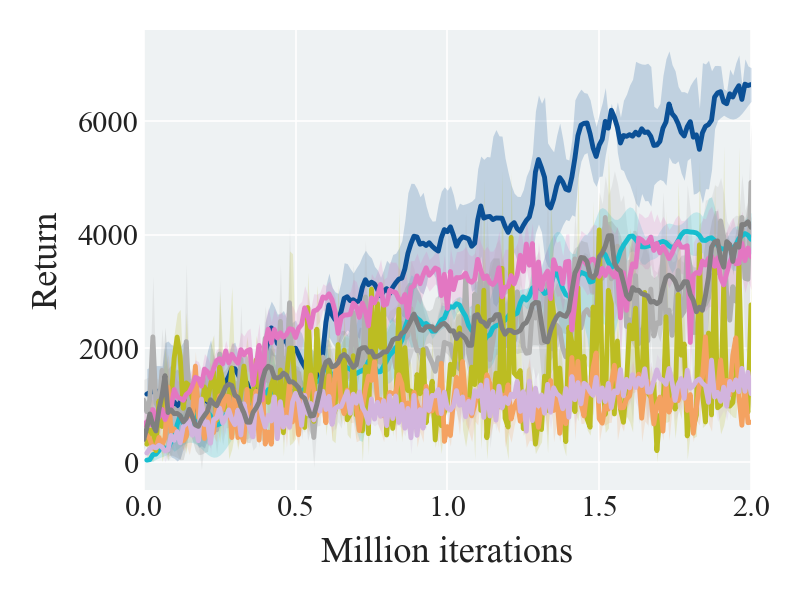}
		\subcaption{H1reach}
		\label{fig:app_h1_reach}
	\end{minipage}
	
	\caption{\textbf{Evaluation curves on benchmarks.} Solid lines denote the mean and shaded regions indicate the 95\% confidence interval.}
	\label{fig:all_tasks}
\end{figure*}

\begin{table*}[t]
	\centering
	\caption{Total Average Return. Performance is reported as the average of the maximum returns observed during the last 10\% of training iterations. \text{Mean} $\pm$ \text{Std} over 5 seeds. The best results are highlighted in \textbf{Bold}.}
	\label{tab:avg_final_return}
	\setlength{\tabcolsep}{1pt}
	\footnotesize
	\resizebox{\textwidth}{!}{
		\begin{tabular}{lccccccccccccc}
			\toprule
			\textbf{Task} & \textbf{Dual-Flow (ours)} & \textbf{FlowRL} & \textbf{QVPO} & \textbf{QSM} & \textbf{BRO} & \textbf{BRO-Fast} & \textbf{SAC} & \textbf{TD3} & \textbf{DDPG} & \textbf{D4PG} & \textbf{MPO} & \textbf{TRPO} & \textbf{A2C} \\
			\midrule
			Humanoid-run   & \textbf{433 $\pm$ 7} & 329 $\pm$ 46 & 1 $\pm$ 1 & 98 $\pm$ 132 & 376 $\pm$ 2 & 364 $\pm$ 1 & 204 $\pm$ 18 & 329 $\pm$ 1 & 2 $\pm$ 1 & 40 $\pm$ 42 & 2 $\pm$ 1 & 2 $\pm$ 1 & 1 $\pm$ 1 \\
			Humanoid-stand & 951 $\pm$ 21 & 960 $\pm$ 3 & 12 $\pm$ 1 & 746 $\pm$ 42 & \textbf{966 $\pm$ 2} & 962 $\pm$ 2 & 964 $\pm$ 2 & 929 $\pm$ 12 & 5 $\pm$ 1 & 7 $\pm$ 1 & 9 $\pm$ 1 & 9 $\pm$ 1 & 8 $\pm$ 1 \\
			Dog-run        & \textbf{628 $\pm$ 33} & 567 $\pm$ 30 & 153 $\pm$ 14 & 138 $\pm$ 114 & 589 $\pm$ 68 & 448 $\pm$ 110 & 585 $\pm$ 21 & 598 $\pm$ 56 & 19 $\pm$ 11 & 106 $\pm$ 41 & 83 $\pm$ 38 & 45 $\pm$ 10 & 22 $\pm$ 1 \\
			Dog-trot       & \textbf{938 $\pm$ 7} & 916 $\pm$ 14 & 96 $\pm$ 1 & 171 $\pm$ 123 & 935 $\pm$ 5 & 935 $\pm$ 2 & 916 $\pm$ 8 & 899 $\pm$ 11 & 16 $\pm$ 7 & 146 $\pm$ 44 & 98 $\pm$ 97 & 49 $\pm$ 12 & 23 $\pm$ 1 \\
			Dog-stand      & 961 $\pm$ 9 & 977 $\pm$ 12 & 590 $\pm$ 86 & 864 $\pm$ 64 & 967 $\pm$ 18 & \textbf{987 $\pm$ 13} & 965 $\pm$ 24 & 973 $\pm$ 4 & 54 $\pm$ 27 & 445 $\pm$ 111 & 390 $\pm$ 95 & 209 $\pm$ 15 & 86 $\pm$ 5 \\
			Dog-walk       & 966 $\pm$ 8 & 941 $\pm$ 11 & 429 $\pm$ 132 & 491 $\pm$ 412 & \textbf{976 $\pm$ 23} & 936 $\pm$ 5 & 944 $\pm$ 10 & 941 $\pm$ 17 & 195 $\pm$ 12 & 245 $\pm$ 18 & 163 $\pm$ 44 & 60 $\pm$ 4 & 27 $\pm$ 4 \\
			Walker-stand   & 990 $\pm$ 4 & 985 $\pm$ 4 & 974 $\pm$ 1 & \textbf{997 $\pm$ 1} & 995 $\pm$ 3 & 995 $\pm$ 4 & 991 $\pm$ 2 & 990 $\pm$ 2 & 985 $\pm$ 4 & 983 $\pm$ 2 & 977 $\pm$ 5 & 302 $\pm$ 77 & 239 $\pm$ 2 \\
			Walker-walk    & \textbf{988 $\pm$ 13} & 977 $\pm$ 1 & 942 $\pm$ 11 & 983 $\pm$ 5 & 978 $\pm$ 4 & 981 $\pm$ 6 & 973 $\pm$ 2 & 976 $\pm$ 1 & 967 $\pm$ 5 & 972 $\pm$ 4 & 917 $\pm$ 46 & 189 $\pm$ 49 & 121 $\pm$ 6 \\
			Walker-run     & 809 $\pm$ 1 & 645 $\pm$ 9 & 612 $\pm$ 24 & 730 $\pm$ 11 & \textbf{818 $\pm$ 0} & 763 $\pm$ 45 & 785 $\pm$ 6 & 794 $\pm$ 8 & 466 $\pm$ 355 & 636 $\pm$ 154 & 392 $\pm$ 48 & 69 $\pm$ 17 & 45 $\pm$ 7 \\
			Quadruped-walk      & 957 $\pm$ 27 & 950 $\pm$ 14 &434 $\pm$ 6& \textbf{986 $\pm$ 7}& 972 $\pm$ 9 & 963 $\pm$ 14 & 949 $\pm$ 20 & 943 $\pm$ 9& 385 $\pm$ 274 & 945 $\pm$ 2 &383 $\pm$ 21& 422 $\pm$  43 & 367 $\pm$ 88   \\
			H1balance\_simple & \textbf{354 $\pm$ 71} & 314 $\pm$ 89 & 221 $\pm$ 1 & 54 $\pm$ 11 & 20 $\pm$ 2 & 16 $\pm$ 2 & 231 $\pm$ 82 & 225$\pm$ 33  & 13 $\pm$ 1 & 12 $\pm$ 1 & 12 $\pm$ 1 & 10 $\pm$ 1 & 13 $\pm$ 1 \\
			H1sit\_hard  &  \textbf{837 $\pm$ 45 }& 728 $\pm$ 60 & 696 $\pm$ 31 & 399 $\pm$ 21 & 17 $\pm$ 3 & 18 $\pm$ 2 & 10 $\pm$ 4 & 14 $\pm$ 15 & 15 $\pm$ 2 & 17 $\pm$ 4 & 11 $\pm$ 4 & 6 $\pm$ 2 &  8 $\pm$ 2 \\
			H1balance\_hard & \textbf{73 $\pm$ 9} & 55 $\pm$ 3 & 57 $\pm$ 8 & 50 $\pm$ 2 & 59 $\pm$ 6 & 47 $\pm$ 0 & 54 $\pm$ 3 & 56 $\pm$ 8 & 46 $\pm$ 5  & 49 $\pm$ 3 &49 $\pm$ 4  & 39 $\pm$ 3 & 46 $\pm$ 2  \\
			H1crawl & \textbf{961 $\pm$ 8} & 884 $\pm$ 11 & 589 $\pm$ 6 & 600 $\pm$ 5 & 428 $\pm$ 4 & 263 $\pm$ 42 & 731 $\pm$ 175 & 607 $\pm$ 117 & 177 $\pm$ 25  & 226 $\pm$ 26 & 212 $\pm$ 25 & 198 $\pm$28  & 161 $\pm$ 25 \\
			H1maze & \textbf{458 $\pm$ 2} & 389 $\pm$ 112 & 304 $\pm$ 9 & 420 $\pm$ 88 & 170 $\pm$ 3 & 140 $\pm$ 11 & 294 $\pm$ 112 & 169 $\pm$ 20 & 118 $\pm$  50  & 156 $\pm$  48 & 145 $\pm$  49 & 143 $\pm$  46 & 114 $\pm$  48 \\
			H1reach & \textbf{6757 $\pm$ 278} & 5179$\pm$ 1789 & 4310 $\pm$ 310 & 4421 $\pm$ 1617 & 2323 $\pm$ 283 & 1855 $\pm$ 848 & 3893 $\pm$ 865 & 4431 $\pm$ 194 & 1686 $\pm$ 809 & 2142 $\pm$ 1008 & 2015 $\pm$ 937 & 1952 $\pm$ 986 & 1616 $\pm$ 751 \\
			\bottomrule
		\end{tabular}%
	}
\end{table*}

\section{Experiments}
\label{exp}
\subsection{Experimental Setup}
\textbf{Baselines}. Our method is compared and evaluated against 13 model-free continuous-control baselines. We consider four generative-policy methods: FlowRL \cite{lv2025flow}, QVPO \cite{ding2024diffusion},  DACER \cite{wang2024diffusion}, QSM \cite{psenka2023learning}; seven well-known model-free RL algorithms: SAC \cite{haarnoja2018soft}, TD3 \cite{fujimoto2018addressing}, D4PG \cite{barth2018distributed}, MPO \cite{abdolmaleki2018maximum}, A2C \cite{mnih2016asynchronous}, TRPO \cite{schulman2015trust}, and DDPG \cite{lillicrap2015continuous}; as well as the state-of-the-art RL variants BRO and BRO-Fast \cite{nauman2024bigger}. All baselines are trained using their official implementations with the default training pipelines.

\textbf{Benchmarks}. We evaluate our method on continuous-control locomotion benchmarks drawn from the DeepMind Control Suite (DMC) \cite{tassa2018deepmind}, Humanoid-Bench (H-Bench) \cite{sferrazza2024humanoidbench}, and MuJoCo Gym (MuJoCo) \cite{brockman2016openai}. All environment settings are provided in Appendix~\ref{env}.
\subsection{Experimental Results}
All curves are presented in Fig.~\ref{fig:all_tasks}, and the detailed numerical results are summarized in Table~\ref{tab:avg_final_return}.  On the DMC Suite, Dual-Flow RL improves over BRO by \textbf{+6.6\%} and FlowRL by \textbf{+10.8\%} on \textit{Dog-run}. For \textit{Humanoid-run}, Dual-Flow RL achieves a clear margin over FlowRL by \textbf{+31.6\%} and over SAC by \textbf{+112.3\%}. Moreover, Dual-Flow RL exhibits faster convergence across these tasks. On the H-Bench benchmark, Dual-Flow RL outperforms FlowRL and QVPO by \textbf{+30.5\%} and \textbf{+56.8\%} on \textit{H1-reach}. These results further confirm the superior empirical performance of Dual-Flow RL.

We additionally compare Dual-Flow RL against strong generative baselines on MuJoCo Gym. As shown in Fig.~\ref{fig:ablation_combined_nav_explore}(a) and~(b), Dual-Flow RL consistently outperforms diffusion- and flow-based methods on both \textit{Humanoid-v3} and \textit{Ant-v3}, demonstrating its advantage over competitive generative policy approaches.

Furthermore, we visualize the predicted return distributions on \textit{Dog-run}.
As shown in Fig.~\ref{fig:dogrun_return_dist}, we consider
C51~\cite{bellemare2017distributional}, IQN~\cite{dabney2018implicit}, and DSAC~\cite{duan2021distributional}. For ground truth, we fix the policy and estimate the distribution by performing repeated rollouts. For fair
comparison, we use 5000 return samples and 60 bins for each histogram. These results demonstrate that our method provides a more accurate characterization of the return distribution.

\begin{algorithm}[t]
	\caption{Dual-Flow RL}
	\label{Dual-Flow}
	\small
	\begin{algorithmic}[1]
		\STATE \textbf{Input:} replay buffer $\mathcal{D}$; discount $\gamma$; Q-flow critic $\theta$ and target $\bar\theta$; policy flow $\varphi$; regulator $\psi$; target rate $\rho$; learning rates $\beta_Q,\beta_\pi,\beta_\sigma$.
		\FOR{each training iteration}
		\FOR{each sampling step}
		\STATE Observe $s$; Sample a flow action $a\sim\pi_{\varphi}(\cdot|s)$.
		\STATE Sample $\epsilon\sim\mathcal{N}(0,I)$ and execute $a_e=\mathrm{clip}\!\left(a+\sigma_{\psi}(s)\odot\epsilon\right)$.
		\STATE Observe $(r,s',d)$ and store $(s,a_e,r,s',d)$ into $\mathcal{D}$.
		\ENDFOR
		\FOR{each update step}
		\STATE Sample mini-batch $\mathcal{B}\sim\mathcal{D}$.
		\STATE $\theta \leftarrow \theta - \beta_Q \nabla_{\theta}\mathcal{L}_{Q}(\theta)$;
		\STATE $\varphi \leftarrow \varphi - \beta_\pi \nabla_{\varphi}\mathcal{L}_{\pi}(\varphi)$;
		\STATE $\psi \leftarrow \psi - \beta_\sigma \nabla_{\psi}\mathcal{L}_{\mathrm{ECER}}(\psi)$
		\STATE $\bar\theta \leftarrow \rho\bar\theta + (1-\rho)\theta$
		\IF{$\mathrm{update}\ \mathrm{mod}\ U = 0$}
		\STATE Update $g_H$, $g_D$, and $\lambda_{\mathrm{eff}}$.
		\ENDIF
		\ENDFOR
		\ENDFOR
	\end{algorithmic}
\end{algorithm}

\subsection{Ablation Studies}
We conduct ablation studies to assess the contribution of each core component of Dual-Flow RL. Specifically, we investigate three questions:
\begin{enumerate}
	\item Does distributional value modeling provide a meaningful advantage over a standard expected-value critic?
	\item How does the proposed ECER influence exploration and overall performance?
	\item How sensitive is the method to key hyperparameters?
\end{enumerate}

\textbf{Distributional Critic and Expected-Value Critic.}
We further evaluate the effect of the value evaluation module on \textit{Dog-trot} and \textit{Humanoid-run} while keeping the policy architecture unchanged.
As shown in Fig.~\ref{fig:ablation_value_flowsteps}(a) and~(b), replacing the expected-value critic with our flow distributional critic leads to faster learning, higher final returns, and more stable performance,
which supports the effectiveness of distributional value modeling for reliable policy improvement.

\begin{figure*}[t]
	\centering
	
	% =========================================================
	% Legends: left for critic ablation, right for flow steps
	% =========================================================
	\begin{minipage}[t]{0.49\textwidth}
		\vspace{0pt}
		\centering
		\includegraphics[width=0.62\linewidth]{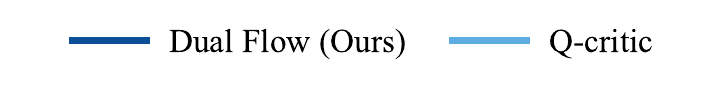}
	\end{minipage}\hfill%
	\begin{minipage}[t]{0.49\textwidth}
		\vspace{0pt}
		\centering
		\includegraphics[width=0.62\linewidth]{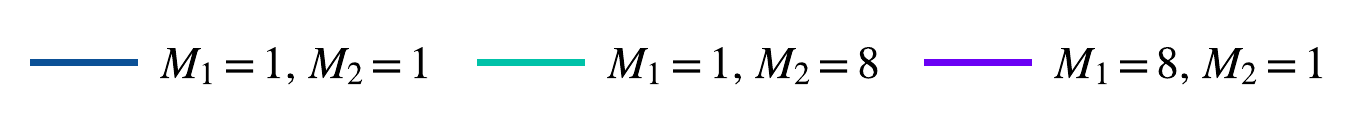}
	\end{minipage}
	
	% =========================================================
	% Left block: critic ablation
	% =========================================================
	\begin{minipage}[t]{0.49\textwidth}
		\vspace{0pt}
		\centering
		
		\begin{subfigure}[t]{0.49\linewidth}
			\centering
			\includegraphics[width=\linewidth]{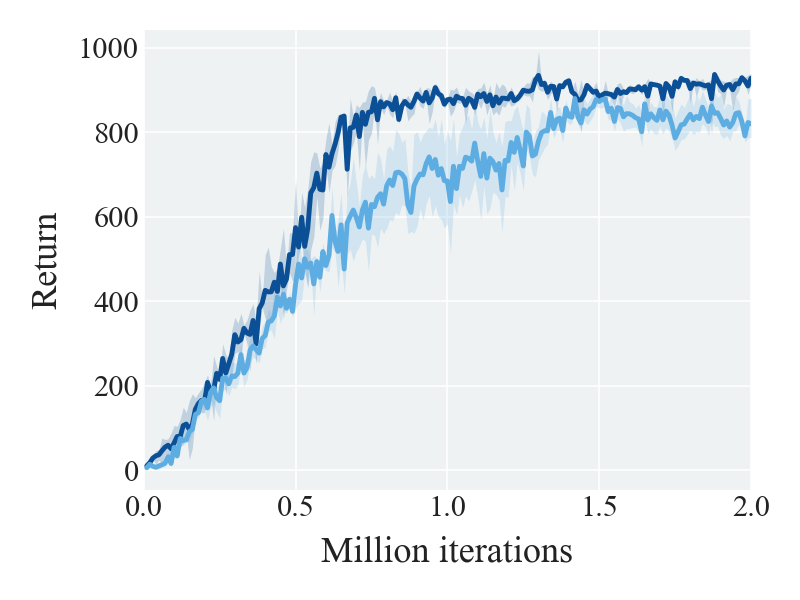}
			\caption{Dog-trot}
			\label{fig:ablation_dogtrot}
		\end{subfigure}\hfill%
		\begin{subfigure}[t]{0.49\linewidth}
			\centering
			\includegraphics[width=\linewidth]{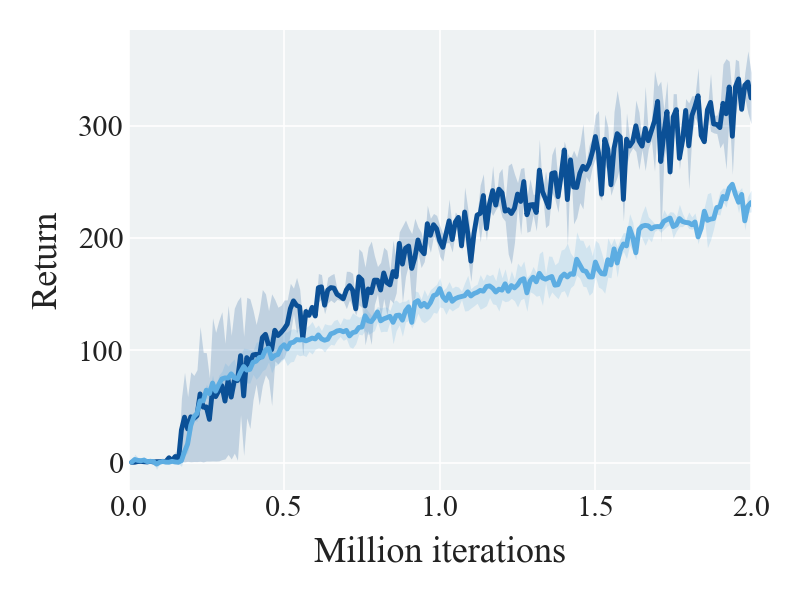}
			\caption{Humanoid-run}
			\label{fig:ablation_value_modeling}
		\end{subfigure}
	\end{minipage}\hfill%
	% =========================================================
	% Right block: flow-step ablation
	% =========================================================
	\begin{minipage}[t]{0.49\textwidth}
		\vspace{0pt}
		\centering
		
		\begin{subfigure}[t]{0.49\linewidth}
			\centering
			\includegraphics[width=\linewidth]{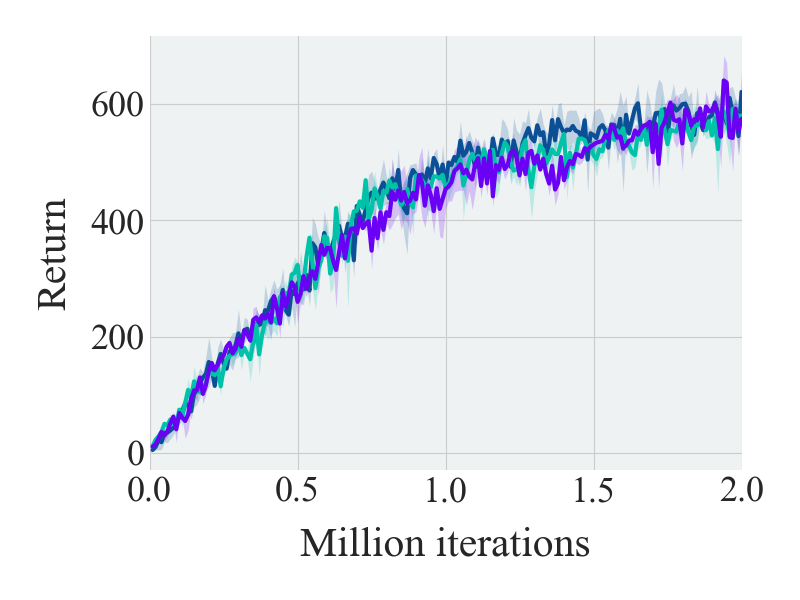}
			\caption{Learning curve}
			\label{fig:flowsteps_curve}
		\end{subfigure}\hfill%
		\begin{subfigure}[t]{0.49\linewidth}
			\centering
			\includegraphics[width=\linewidth]{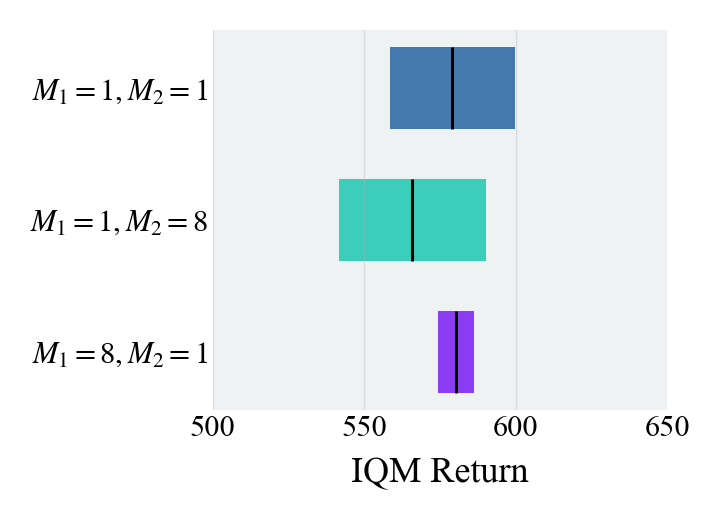}
			\caption{IQM}
			\label{fig:flowsteps}
		\end{subfigure}
	\end{minipage}
	
	\caption{
		\textbf{Ablation studies.}
		\textbf{(a)--(b) Distributional critic vs.\ expected-value critic:}
		Replacing the expected-value critic with our flow distributional critic
		leads to faster learning and higher final returns on
		\textit{Dog-trot} (a) and \textit{Humanoid-run} (b).
		\textbf{(c)--(d) Sensitivity to flow steps on \textit{Dog-run}:}
		Learning curves (c) and IQM (d) across different $M_1$ and $M_2$
		settings show that the method is robust to the choice of flow-step counts.
	}
	\label{fig:ablation_value_flowsteps}
\end{figure*}

\textbf{Exploration Regulator.} We compare different regulation mechanisms to assess the advantage of our ECER. Concretely, we consider the DACER entropy regulator~\cite{wang2024diffusion} and initial noise-scale tuning~\cite{jain2024sampling}.
For the DACER regulator, we isolate its entropy-based exploration module and integrate it into our framework: the policy output is augmented by a learned diagonal noise $\sigma_\text{out}(t) = \lambda_\text{noise}\cdot\alpha(t)$, where $\alpha(t)$ is updated via a proportional entropy feedback loop targeting $H_\text{tgt}=-c\cdot d$. For initial noise-scale tuning, we replace $a_0\sim\mathcal{N}(0,I)$ with $a_0\sim\mathcal{N}(0,\xi I)$ and sweep $\xi$.

As shown in Fig.~\ref{fig:ablation_combined_nav_explore}(c) and~(d), the DACER regulator brings modest gains, while initial noise-scale tuning yields little improvement. In contrast, ECER achieves higher final returns. We attribute this benefit to ECER's state-aware closed-loop regulation: through state-aware exploration regulation, ECER injects noise more effectively into high-value action regions, improving the quality of target samples for flow policy training and yielding consistent gains.

\textbf{Flow steps.}
We examine the method's sensitivity to flow steps on \textit{Dog-run} by independently varying the step counts $M_1$, $M_2$ for the flow-based critic and flow-based policy, respectively. As shown in Fig.~\ref{fig:ablation_value_flowsteps}(c) and~(d), the resulting learning curves are largely indistinguishable across configurations, exhibiting similar convergence behavior and final performance. In contrast, increasing the number of flow steps yields longer backpropagation-through-time (BPTT) chains through the time-unrolled flow transformations, which substantially increases computational and memory costs and prolongs training. Overall, our method appears robust to the choice of flow-step counts, and shallow inference is generally sufficient in practice for stable and efficient learning.   

\textbf{Regularization coefficient $\lambda_0$.}
We analyze sensitivity to the base regularization parameter $\lambda_0$ in ECER, which sets the baseline penalty on the exploration scale. We sweep $\lambda_0 \in \{0.05, 0.1, 0.5, 1, 5 \}$ on Humanoid-run and Dog-run; results are shown in Fig.~\ref{fig:lambda_sensitivity}. Large values restrict exploration and degrade performance, while excessively small values introduce nuisance variance. The method is robust within a moderate range.

\begin{figure*}[t]
	\centering

	\includegraphics[width=0.68\textwidth]{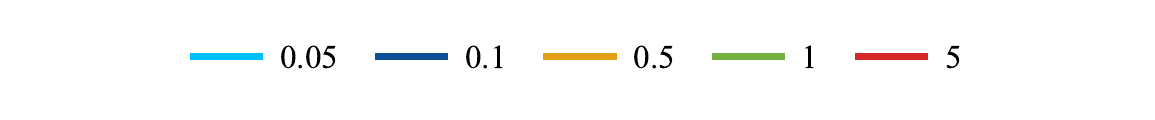}\\[-0.3em]

	\begin{subfigure}[b]{0.245\textwidth}
		\centering
		\includegraphics[width=\textwidth]{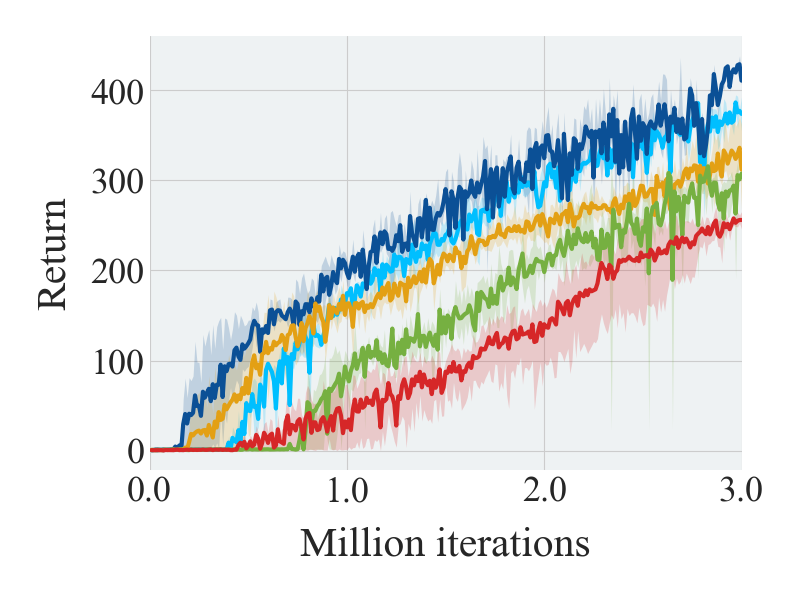}
		\caption{Learning curve}
		\label{fig:run_lamba}
	\end{subfigure}\hfill
	\begin{subfigure}[b]{0.245\textwidth}
		\centering
		\includegraphics[width=\textwidth]{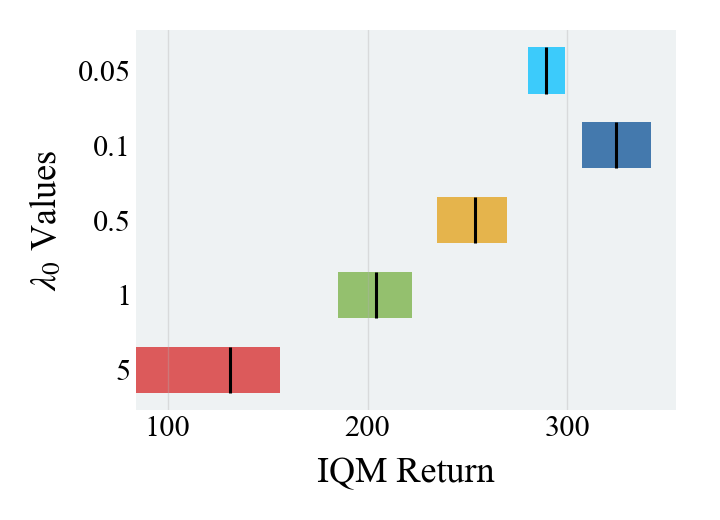}
		\caption{IQM Return}
		\label{fig:lambda_humanoid_iqm}
	\end{subfigure}\hfill
	\begin{subfigure}[b]{0.245\textwidth}
		\centering
		\includegraphics[width=\textwidth]{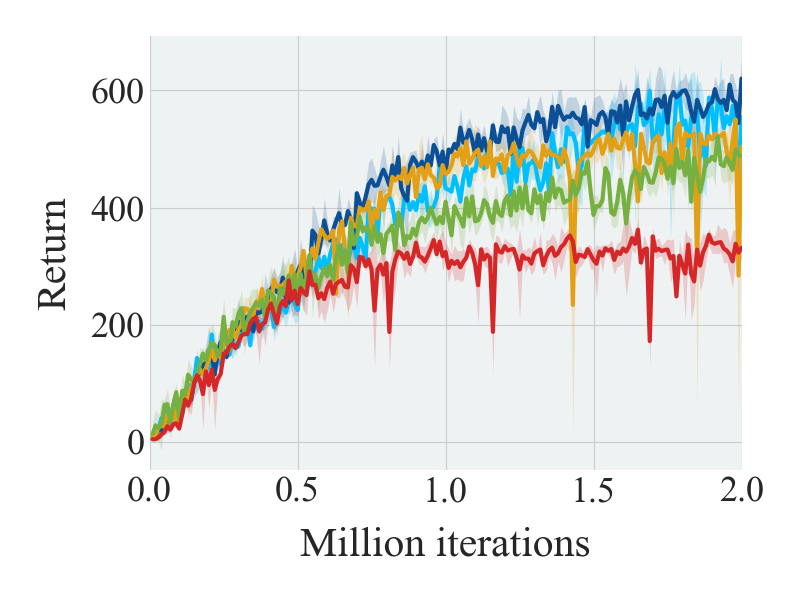}
		\caption{Learning curve}
		\label{fig:lambda_dog_curve}
	\end{subfigure}\hfill
	\begin{subfigure}[b]{0.245\textwidth}
		\centering
		\includegraphics[width=\textwidth]{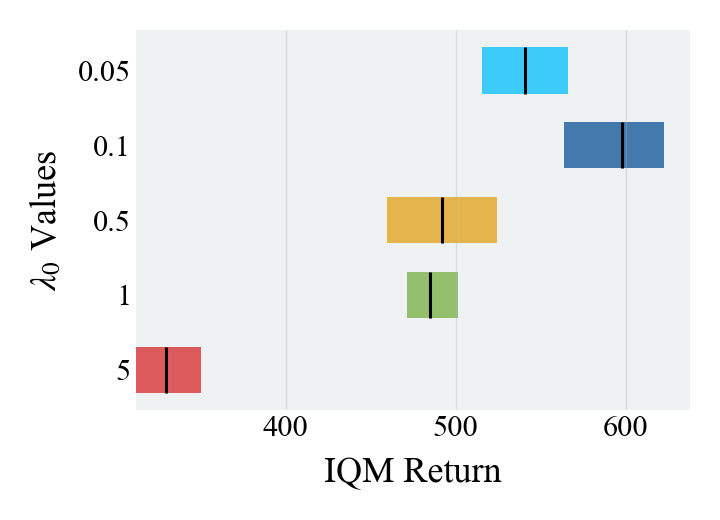}
		\caption{IQM Return}
		\label{fig:lambda_dog_iqm}
	\end{subfigure}

	\vspace{-0.3em}
	\caption{\textbf{Sensitivity analysis to the base regularization parameter $\lambda_0$.}
		(a,c) show learning curves on \textit{Humanoid-run} and \textit{Dog-run} for different $\lambda_0$. (b,d) report the corresponding aggregated final performance (IQM return) with 95\% stratified bootstrap confidence intervals. Larger $\lambda_0$ (e.g., $1,5$) restricts exploration and degrades performance, while smaller values (e.g., $0.05,0.1$) behave similarly.}
	\label{fig:lambda_sensitivity}
\end{figure*}

\section{Conclusion and Future Work}
We propose Dual-Flow RL, a unified flow-based actor-critic framework that jointly parameterizes the policy and the return distribution via CFM, aiming to provide more accurate and reliable value estimates while promoting multimodal exploration. To further strengthen exploration in online learning, we introduce an Entropy-Covariance Exploration Regulator (ECER) that enables state-aware exploration. Experiments on three benchmarks demonstrate clear improvements over strong baselines. Future work will focus on improving value-estimation efficiency and extending the framework to more challenging, risk-sensitive control domains.

\appendices

\section{GMM-Based Entropy Estimation}
\label{Derivation-1}
To obtain a tractable and differentiable estimate of policy entropy in continuous action spaces, we periodically approximate the state-conditioned action distribution of the current policy $\pi_\phi(\cdot \mid s)$ using a Gaussian Mixture Model (GMM) with diagonal covariances. For each state $s$ sampled from a batch $\{s_b\}_{b=1}^B$, we draw $N$ action samples $\{a^{(i)}\}_{i=1}^N$ from the current policy to fit a $K_m$-component diagonal GMM:
\begin{equation}
	\begin{aligned}
		&q(a \mid s) = \sum_{k_m=1}^{K_m} \pi_{k_m}(s) \mathcal{N}(a; \mu_{k_m}(s), \text{diag}(v_{k_m}(s))), \quad \\
		&\sum_{k_m=1}^{K_m} \pi_{k_m}(s) = 1, \pi_{k_m}(s) \geq 0
	\end{aligned}
\end{equation}
where $\pi_{k_m}(s)$ denotes the mixture weights, $\mu_{k_m}(s) \in \mathbb{R}^d$ is the mean vector, and $v_{k_m}(s) \in \mathbb{R}^d_{+}$ represents the diagonal variance vector (i.e., $\Sigma_{k_m}(s) = \text{diag}(v_{k_m}(s))$) for a $d$-dimensional action space.

\textbf{Analytical Approximation of State-Conditioned Entropy.}
Since the differential entropy of a GMM lacks a closed-form expression, we adopt a standard approximation~\cite{huber2008entropy}: introducing a component indicator $Z$ with $P(Z=k_m \mid s)=\pi_{k_m}(s)$, the entropy decomposes as:
\begin{equation}
	\begin{aligned}
		&H(s) \approx -\sum_{k_m=1}^{K_m} \pi_{k_m}(s) \log \pi_{k_m}(s) + \\ &\sum_{k_m=1}^{K_m} \pi_{k_m}(s) H(\mathcal{N}(\mu_{k_m}(s), \text{diag}(v_{k_m}(s)))).
	\end{aligned}
\end{equation}

\textbf{Simplification for Diagonal Gaussian Components.}
For a diagonal Gaussian, $|\mathrm{diag}(v)| = \prod_{j=1}^d v_j$, so the component entropy reduces to:
\begin{equation}
	\begin{aligned}
		&H(\mathcal{N}(\mu, \text{diag}(v))) = \frac{1}{2} \log \left( (2\pi e)^d \prod_{j=1}^{d} v_j \right) \\
		&= \frac{1}{2} \sum_{j=1}^{d} \log(2\pi e v_j).
	\end{aligned}
\end{equation}
Substituting into the approximation above yields the final entropy estimate:
\begin{equation}
	\begin{aligned}
		&H(s) \approx -\sum_{k_m=1}^{K_m} \pi_{k_m}(s) \log \pi_{k_m}(s) + \\ &\frac{1}{2} \sum_{k_m=1}^{K_m} \pi_{k_m}(s) \sum_{j=1}^{d} \log(2\pi e v_{k_m,j}(s)).
	\end{aligned}
\end{equation}
To estimate the global policy entropy over a batch of states $\{s_b\}_{b=1}^B$, we compute the individual $H(s_b)$ for each state and take the empirical mean:
\begin{equation}
	\begin{aligned}
		\hat{H} \triangleq \frac{1}{B} \sum_{b=1}^{B} H(s_b).
	\end{aligned}
\end{equation}
The estimator $\hat{H}$ characterizes the average exploration intensity of the current policy and serves as a primary input signal for subsequent adaptive exploration tuning mechanisms.

\section{Proof of Proposition 4.3 (Distributional Bellman Consistency at the Optimum)}
\label{proof-4.1}
Fix $(s, a)$, all arguments below are conditional on $(s, a)$.
\begin{lemma}[Optimal velocity under CFM]
\label{lem:1}
Notation simplification. For readability, we denote the optimal return-axis velocity field by $v^*$, i.e., $v^* \triangleq v^*_z$.

For any fixed $(t,z)$, define the conditional objective
\begin{equation}
	\begin{aligned}
		&J(v;t,z)=\mathbb{E}\big[(v-u)^2\mid t,\ z_t=z\big],\qquad v\in\mathbb{R}, \\
		 &z_t = (1 - t)z_0 + t z_1, \quad t \in [0, 1].
	\end{aligned}
\end{equation}
Then
\begin{equation}
	\begin{aligned}
		v^*(t,z)=\arg\min_{v\in\mathbb{R}}\,J(v;t,z)=\mathbb{E}\big[u\mid t,\ z_t=z\big],
	\end{aligned}
\end{equation}
and the minimizer is unique.
\end{lemma}

\begin{IEEEproof}
Expand the square:
\begin{equation}
	\begin{aligned}
		(v-u)^2=v^2-2vu+u^2.
	\end{aligned}
\end{equation}
Taking conditional expectation given $(t,z_t=z)$ yields
\begin{equation}
	\begin{aligned}
		&J(v;t,z)=\mathbb{E}[v^2-2vu+u^2\mid t,z_t=z] \\
		&=v^2-2v\,\mathbb{E}[u\mid t,z_t=z]+\mathbb{E}[u^2\mid t,z_t=z].
	\end{aligned}
\end{equation}
The term $\mathbb{E}[u^2\mid t,z_t=z]$ is constant with respect to $v$, hence $J(\cdot;t,z)$ is a quadratic function of $v$. Differentiate:
\begin{equation}
	\begin{aligned}
		\frac{d}{dv}J(v;t,z)=2v-2\,\mathbb{E}[u\mid t,z_t=z].
	\end{aligned}
\end{equation}
Setting the derivative to zero gives the unique stationary point
\begin{equation}
	\begin{aligned}
		v=\mathbb{E}[u\mid t,z_t=z].
	\end{aligned}
\end{equation}
Moreover, the second derivative satisfies
\begin{equation}
	\begin{aligned}
		\frac{d^2}{dv^2}J(v;t,z)=2>0.
	\end{aligned}
\end{equation}
Therefore, $J(\cdot;t,z)$ is strictly convex and the stationary point is the unique global minimizer. Therefore
\begin{equation}
	\begin{aligned}
		v^*(t,z)=\arg\min_{v\in\mathbb{R}}\,\mathbb{E}[(v-u)^2\mid t,z_t=z]=\mathbb{E}[u\mid t,z_t=z].
	\end{aligned}
\end{equation}
Thus, the optimal velocity field at each fixed $(t, z)$ is given by the conditional expectation $v^*(t, z) = \mathbb{E}[u \mid t, z_t = z]$, and its uniqueness is guaranteed by the strict convexity of the loss function.
\end{IEEEproof}

\begin{lemma}[Weak transport of interpolation marginals]
\label{lem:2}
Let $\mu_t$ denote the marginal law of $z_t=(1-t)z_0+t z_1$. Then for any test function $\varphi\in C_c^\infty(\mathbb{R})$,
\begin{equation}
	\begin{aligned}
		\frac{d}{dt}\int \varphi(z)\,d\mu_t(z)=\int \varphi'(z)\,v^*(t,z)\,d\mu_t(z),\qquad \mu_0=p_0.
	\end{aligned}
\end{equation}
\end{lemma}

\begin{IEEEproof}
Define $F(t)\triangleq \mathbb{E}[\varphi(z_t)]$. Since $z_t=(1-t)z_0+t z_1$ is differentiable in $t$ for each fixed $(z_0,z_1)$, we have
\begin{equation}
	\begin{aligned}
		\frac{dz_t}{dt}=z_1-z_0=u.
	\end{aligned}
\end{equation}
By the chain rule, for each realization,
\begin{equation}
	\begin{aligned}
		\frac{d}{dt}\varphi(z_t)=\varphi'(z_t)\frac{dz_t}{dt}=\varphi'(z_t)u.
	\end{aligned}
\end{equation}
Under standard integrability conditions, we can interchange the derivative and expectation to obtain
\begin{equation}
	\begin{aligned}
		F'(t)=\frac{d}{dt}\mathbb{E}[\varphi(z_t)]=\mathbb{E}\Big[\frac{d}{dt}\varphi(z_t)\Big]=\mathbb{E}[\varphi'(z_t)u].
	\end{aligned}
\end{equation}
Apply the tower property conditioning on $(t, z_t)$:
\begin{equation}
	\begin{aligned}
		\mathbb{E}[\varphi'(z_t)u]=\mathbb{E}\big[\,\mathbb{E}[\varphi'(z_t)u\mid t,z_t]\,\big].
	\end{aligned}
\end{equation}
The inner conditional expectation $\mathbb{E}[\varphi'(z_t)u \mid t, z_t]$ averages over the randomness in $(z_0, z_1)$ given $(t, z_t)$, aggregating contributions from all endpoint pairs that satisfy the interpolation condition $z_t = (1 - t)z_0 + t z_1$. The outer expectation then averages this quantity over the marginal law of $(t, z_t)$, where $z_t$ follows $\mu_t$.

Since $\varphi'(z_t)$ is a deterministic function of $z_t$, it is fixed once $(t, z_t)$ is given. Hence it can be taken outside the conditional expectation:
\begin{equation}
	\begin{aligned}
		\mathbb{E}[\varphi'(z_t)u\mid t,z_t]=\varphi'(z_t)\,\mathbb{E}[u\mid t,z_t].
	\end{aligned}
\end{equation}
Therefore
\begin{equation}
	\begin{aligned}
		F'(t)=\mathbb{E}[\varphi'(z_t)\,\mathbb{E}[u\mid t,z_t]]=\mathbb{E}[\varphi'(z_t)v^*(t,z_t)],
	\end{aligned}
\end{equation}
where the last equality uses Lemma 1: $\mathbb{E}[u\mid t,z_t]=v^*(t,z_t)$. Rewriting this expectation using the marginal $\mu_t$ gives
$$F'(t)=\int \varphi'(z)\,v^*(t,z)\,d\mu_t(z).$$
Finally, $\mu_0$ is the law of $z_{t=0}=z_0$, hence $\mu_0=p_0$.

In summary, Lemma 2 proves that the interpolation $z_t = (1 - t)z_0 + tz_1$ induces marginals $\{\mu_t\}_{t\in[0,1]}$ that satisfy the weak transport equation $\frac{d}{dt} \int \varphi(z) \, d\mu_t(z) = \int \varphi'(z) \, v^*(t, z) \, d\mu_t(z)$. By construction, the interpolation endpoints match $\mu_1 = p_{td}$.
\end{IEEEproof}

\begin{lemma}[Weak transport of the ODE flow]
\label{lem:3}
Let $z(t)$ solve the ODE
\begin{equation}
	\begin{aligned}
		\frac{dz}{dt}=v^*(t,z),\qquad z(0)\sim p_0,
	\end{aligned}
\end{equation}
and let $\nu_t$ be the marginal law of $z(t)$. Then for any $\varphi\in C_c^\infty(\mathbb{R})$,
\begin{equation}
	\begin{aligned}
		\frac{d}{dt}\int \varphi(z)\,d\nu_t(z)=\int \varphi'(z)\,v^*(t,z)\,d\nu_t(z),\qquad \nu_0=p_0.
	\end{aligned}
\end{equation}
\end{lemma}

\begin{IEEEproof}
Define $G(t)\triangleq \mathbb{E}[\varphi(z(t))]=\int \varphi(z)\,d\nu_t(z)$. For each sample path, the chain rule gives
\begin{equation}
	\begin{aligned}
		\frac{d}{dt}\varphi(z(t))=\varphi'(z(t))\frac{dz(t)}{dt}.
	\end{aligned}
\end{equation}
Using the ODE $\frac{dz(t)}{dt}=v^*(t,z(t))$, we obtain
\begin{equation}
	\begin{aligned}
		\frac{d}{dt}\varphi(z(t))=\varphi'(z(t))\,v^*(t,z(t)).
	\end{aligned}
\end{equation}

For any test function $\varphi \in C_c^\infty$, we can differentiate under the expectation; it follows that
\begin{equation}
	\begin{aligned}
		G'(t)=\mathbb{E}\Big[\frac{d}{dt}\varphi(z(t))\Big]=\mathbb{E}[\varphi'(z(t))v^*(t,z(t))].
	\end{aligned}
\end{equation}
Rewriting the expectation with respect to the marginal $\nu_t$ yields
\begin{equation}
	\begin{aligned}
		G'(t)=	\frac{d}{dt}\int \varphi(z)\,d\nu_t(z)=\int \varphi'(z)\,v^*(t,z)\,d\nu_t(z).
	\end{aligned}
\end{equation}
The initial condition $\nu_0=p_0$ holds because $z(0)\sim p_0$.

In summary, Lemma 3 shows that the ODE-induced marginals $\{\nu_t\}$ satisfy the same weak transport identity as in Lemma 2, with $\nu_0 = p_0$.
\end{IEEEproof}

\begin{lemma}[Uniqueness of weak solutions via the backward transport equation]
\label{lem:4}
Under Assumption 4.2, the weak solution curve of the transport equation driven by $v^*(t, z)$ is unique: if two probability-measure curves $\{\eta_t\}$ and $\{\tilde{\eta}_t\}$ satisfy, for all $\varphi \in C_c^\infty(\mathbb{R})$,
\begin{equation}
	\begin{aligned}
		&\frac{d}{dt} \int \varphi(z) \, d\eta_t(z) = \int \varphi'(z) v^*(t, z) \, d\eta_t(z), \quad \\
		&\frac{d}{dt} \int \varphi(z) \, d\tilde{\eta}_t(z) = \int \varphi'(z) v^*(t, z) \, d\tilde{\eta}_t(z),
	\end{aligned}
\end{equation}
and $\eta_0 = \tilde{\eta}_0$, then $\eta_t = \tilde{\eta}_t$ for all $t \in [0, 1]$.
\end{lemma}

\begin{IEEEproof}
Fix any $T\in(0,1]$ and any $\varphi\in C_c^\infty(\mathbb{R})$.
Let $\psi$ denote the solution to the backward transport equation
\begin{equation}
	\begin{aligned}
		\partial_t \psi(t, z) + v^*(t, z)\partial_z \psi(t, z) = 0, \quad \psi(T, z) = \varphi(z),
	\end{aligned}
\end{equation}
whose existence and uniqueness are guaranteed by Assumption 4.2.

First, we define $$H(t)\triangleq \int \psi(t,z)\,d\eta_t(z).$$ We compute $H'(t)$.
Because $\psi$ depends on $t$ explicitly and $\eta_t$ varies with $t$, we write
\begin{equation}
	\begin{aligned}
		&\frac{d}{dt}\int \psi(t,z)\,d\eta_t(z)=\int \partial_t\psi(t,z)\,d\eta_t(z)+ \\
		&\frac{d}{dt}\int \psi(t,z)\,d\eta_t(z)\Big|_{\psi(t,\cdot) \text{fixed}}.
	\end{aligned}
\end{equation}
For the second term, treat $\psi(t,\cdot)$ as a test function in $z$ at fixed $t$. Applying the weak identity with $\varphi=\psi(t,\cdot)$ gives
\begin{equation}
	\begin{aligned}
		\frac{d}{dt}\int \psi(t,z)\,d\eta_t(z)\Big|_{\psi(t,\cdot) \text{fixed}}=\int \partial_z\psi(t,z)\,v^*(t,z)\,d\eta_t(z).
	\end{aligned}
\end{equation}
Hence
\begin{equation}
	\begin{aligned}
		&H'(t)=\int \partial_t\psi(t,z)\,d\eta_t(z)+\int \partial_z\psi(t,z)\,v^*(t,z)\,d\eta_t(z) \\
		&=\int\big(\partial_t\psi+v^*\partial_z\psi\big)\,d\eta_t.
	\end{aligned}
\end{equation}
Since $\psi$ solves the backward PDE, $\partial_t\psi+v^*\partial_z\psi=0$, and thus $H'(t)=0$. Therefore $H(t)$ is constant and $H(T)=H(0)$.
Next, we relate the terminal and initial expectations.
Using $\psi(T,z)=\varphi(z)$, we have
\begin{equation}
	\begin{aligned}
		&H(T)=\int \psi(T,z)\,d\eta_T(z)=\int \varphi(z)\,d\eta_T(z),\\
		&H(0)=\int \psi(0,z)\,d\eta_0(z).
	\end{aligned}
\end{equation}
Thus
\begin{equation}
	\begin{aligned}
		\int \varphi(z)\,d\eta_T(z)=\int \psi(0,z)\,d\eta_0(z).
	\end{aligned}
\end{equation}
Applying the same argument to $\tilde\eta_t$ yields
\begin{equation}
	\begin{aligned}
		\int \varphi(z)\,d\tilde\eta_T(z)=\int \psi(0,z)\,d\tilde\eta_0(z).
	\end{aligned}
\end{equation}
If $\eta_0=\tilde\eta_0$, then for all $\varphi\in C_c^\infty(\mathbb{R})$,
\begin{equation}
	\begin{aligned}
		\int \varphi(z)\,d\eta_T(z)=\int \varphi(z)\,d\tilde\eta_T(z).
	\end{aligned}
\end{equation}
Finally, since two probability measures agree on all test functions $\varphi\in C_c^\infty(\mathbb{R})$, then the measures are equal. Hence $\eta_T=\tilde\eta_T$. Since $T$ is arbitrary, $\eta_t=\tilde\eta_t$ for all $t\in[0,1]$.
\end{IEEEproof}

By Lemma 2 and Lemma 3, both $\mu_t$ (interpolation marginals) and $\nu_t$ (ODE pushforward marginals) satisfy the same weak transport identity with the same initial condition $p_0$. By Lemma 4 (uniqueness), it follows that
\begin{equation}
	\begin{aligned}
		\mu_t=\nu_t\quad \forall\,t\in[0,1].
	\end{aligned}
\end{equation}
In particular, $\nu_1=\mu_1$. Under the interpolation construction, $z_{t=1}=z_1\sim p_{\mathrm{td}}$, hence $\mu_1=p_{\mathrm{td}}$. Therefore $\nu_1=p_{\mathrm{td}}$, i.e., the probability-flow ODE induced by the globally optimal CFM field transports $p_0$ to $p_{\mathrm{td}}$.  Since $p_{\text{td}}(\cdot \mid s, a) \doteq (\mathcal{T}^\pi Z)(\cdot \mid s, a)$ and $Z(\cdot \mid s, a)$ is the distribution obtained by transporting $p_0$ through ODE, we have $Z(\cdot \mid s, a) = (\mathcal{T}^\pi Z)(\cdot \mid s, a)$.

\section{Implementation Details}
\label{env}
\textbf{Environments:} We evaluate our method on continuous-control benchmarks from the DeepMind Control Suite (DMC), Humanoid-Bench (H-Bench), and MuJoCo Gym; all environments are visualized in Fig.~\ref{fig:all_environments}. All experiments use four NVIDIA RTX 3090 GPUs and results are averaged over five random seeds.

\begin{figure*}[t]
	\centering
	\includegraphics[width=0.95\textwidth]{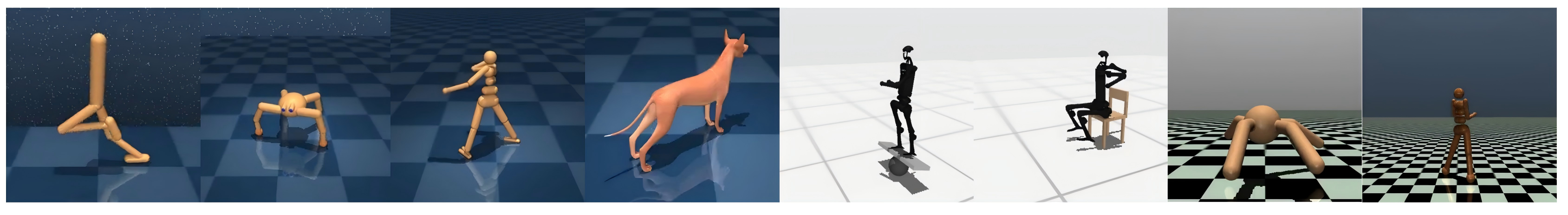}

\caption{\textbf{Visualizations of the considered environments.}
	From the DeepMind Control Suite (DMC)~\cite{tassa2018deepmind}, we include ten locomotion tasks: \textit{Humanoid} (Run/Stand), \textit{Dog} (Run/Trot/Stand/Walk), \textit{Walker} (Run/Walk/Stand), and \textit{Quadruped} (Walk).
	From H-Bench~\cite{sferrazza2024humanoidbench}, we include \textit{H1-sit\_hard} and \textit{H1-balance\_simple}.
	From MuJoCo Gym~\cite{brockman2016openai}, we include \textit{Humanoid-v3} and \textit{Ant-v3}.
	These environments are characterized by high-dimensional state-action spaces and complex contact dynamics.}
	\label{fig:all_environments}
\end{figure*}

%\textbf{Policy Flow Network:} The policy is parameterized as a time-conditioned velocity field $v_\phi$, taking $[s, a, t] \in \mathbb{R}^{d_s+d_a+1}$ as input, where $t$ is a raw scalar without Fourier embeddings. The architecture is a 3-layer MLP with two hidden layers of width $H=256$, each followed by Layer Normalization and ELU, and a linear output layer in $\mathbb{R}^{d_a}$.
%
%\textbf{Flow Distributional Critic Network:} Implemented as a one-dimensional conditional flow with velocity field $v_\theta$, taking $[s, a, z, t] \in \mathbb{R}^{d_s+d_a+2}$ as input ($t$ as raw scalar). The same 3-layer MLP structure (width $H=256$, LayerNorm, ELU) produces a scalar output, directly parameterizing the velocity field of the conditional return distribution rather than its expectation or quantiles.
\textbf{Policy Flow Network:} The policy is parameterized as a time-conditioned velocity field $v_\phi$, with parameters independent of all other modules. The network takes as input the concatenation of state, action, and time, $[s, a, t] \in \mathbb{R}^{d_s+d_a+1}$, where the time variable $t$ is directly concatenated as a raw scalar without any additional time embeddings or Fourier features. The architecture is a three-layer fully connected MLP with two hidden layers of width $H = 256$, each followed by Layer Normalization and an ELU activation, and a linear output layer producing a vector in $\mathbb{R}^{d_a}$. The network outputs an action-space velocity field.

\textbf{Flow Distributional Critic Network:} The flow distributional critic is implemented as a one-dimensional conditional flow, with its velocity field network $v_\theta$ parameterized independently from the policy. The input consists of the concatenation of state, action, return variable, and time, $[s, a, z, t] \in \mathbb{R}^{d_s+d_a+2}$, where the time variable is again provided as a raw scalar. The network is a three-layer fully connected MLP with two hidden layers of width $H = 256$, each followed by Layer Normalization and ELU activations, and a linear output layer producing a single scalar. This architecture directly parameterizes the velocity field of the conditional return distribution, rather than its expectation or quantiles.

\textbf{Entropy-Covariance Exploration Regulator:} The ECER is implemented as a state-conditioned exploration scale network $\sigma_\psi(s)$, with parameters independent of both the policy and value networks. The network takes $s \in \mathbb{R}^{d_s}$ as input and employs a 3-layer MLP with two hidden layers of width $H=256$, each followed by Layer Normalization and ELU, and a linear output layer producing a vector of dimension $d_a$. The output is mapped via an exponential function and clipped to a preset range, serving as the per-dimension exploration scale to modulate the injected noise intensity. All hyperparameters used across experiments are summarized in Table~\ref{tab:dualflow_hparams_block}.

\begin{table*}[t]
	\centering
	\caption{Hyperparameter settings for \textbf{Dual-Flow RL} (ours).}
	\label{tab:dualflow_hparams_block}
	\small
	\setlength{\tabcolsep}{6pt}
	\renewcommand{\arraystretch}{1.08}
	\begin{tabular}{l c | l c}
		\toprule
		\textbf{Hyperparameter} & \textbf{Value} & \textbf{Hyperparameter} & \textbf{Value} \\
		\midrule
		\multicolumn{4}{l}{\textbf{Training}} \\
		\midrule
		Discount $\gamma$ & 0.99 & Polyak rate $\tau$ & 0.005 \\
		Replay buffer size $|\mathcal{D}|$ & 1e6 & Batch size $B$ & 256 \\
		Optimizer & Adam & Learning rate $(\beta_Q,\beta_\pi)$ & $3\times 10^{-4}$ \\
		Target entropy $\mathcal{H}_{\rm tgt}$ & $-2.2\,\mathrm{dim}(\mathcal{A})$ & Warmup updates $T_{\text{distill}}$ & 2e5 \\
		Update interval $U$ of $\lambda_{\mathrm{eff}}$ & 1e4 & Warmup updates $T_{\text{ECER}}$ & 2e5\\
		\midrule
		\multicolumn{4}{l}{\textbf{Network Architecture}} \\
		\midrule
		Critic depth / width & 2 / 256 & Policy depth / width & 2 / 256 \\
		Critic flow steps $M_1$ & 1 & Policy flow steps $M_2$ & 1 \\
		Prior distribution $p_0$ & $\mathcal{N}(0,I)$ & Samples $K$ & 16 \\
		\midrule
		\multicolumn{4}{l}{\textbf{Entropy-Covariance Exploration Regulator (ECER)}} \\
		\midrule
		Entropy update interval $U$ & 1e4 & GMM state batch size $B_H$ & 32 \\
		Actions per state $N$ & 200 & GMM components $K_m$ & 3 \\
		Entropy EMA $\beta_H$ & 0.95 & EMA $\beta_{\rho}$ & 0.97 \\
		Cov gate coeff.\ $k_D$ & 10.0 & Cov gate cap $g_{D,\max}$ & 2.0 \\
		Threshold $\rho_{\rm thr}$ & 0.10 & Threshold $\Delta\rho_{\rm thr}$ & 0.003 \\
		Combined gate cap $g_{\max}$ & 3.0 & $\lambda_0$ & 0.1 \\
		\bottomrule
	\end{tabular}
\end{table*}

%\section*{Acknowledgment}
%The authors would like to thank the anonymous reviewers for their constructive comments.

\bibliographystyle{IEEEtran}
\bibliography{refs}

\end{document}